\documentclass{article}
\usepackage{fullpage}
\usepackage{microtype}
\usepackage{graphicx}
\usepackage{subfigure}
\usepackage{booktabs} %
\usepackage{adjustbox}
\usepackage{makecell}
\usepackage{color-edits}
\addauthor{vs}{red}
\usepackage{authblk}
\usepackage[noend]{algorithmic}
\usepackage[ruled,noline,noend]{algorithm2e}

\usepackage[round]{natbib}

\usepackage{hyperref}

\usepackage{amsmath}
\usepackage{amssymb}
\usepackage{mathtools}
\usepackage{amsthm}

\usepackage[capitalize,noabbrev]{cleveref}

\newcommand{\cH}{{\mathcal{H}}}
\newcommand{\Hcal}{{\mathcal{H}}}
\newcommand{\cF}{{\mathcal{F}}}
\newcommand{\cD}{{\mathcal D}}

\newcommand{\cX}{{\mathcal X}}
\newcommand{\cY}{{\mathcal Y}}
\newcommand{\cR}{{\mathcal R}}
\newcommand{\cK}{{\mathcal K}}
\newcommand{\cM}{{\mathcal M}}
\newcommand{\cB}{{\mathcal B}}
\newcommand{\bE}{\mathop{\mathbb{E}}}
\newcommand{\R}{{\mathbb{R}}}
\newcommand{\shull}{\ensuremath{\mathrm{star}}}

\newcommand{\cG}{{\mathcal{G}}}

\newcommand{\argmin}{\arg\min}

\newcommand{\hz}{{h_0}}

\theoremstyle{plain}
\newtheorem{theorem}{Theorem}[section]

\newtheorem{lemma}[theorem]{Lemma}
\newtheorem{corollary}[theorem]{Corollary}
\theoremstyle{definition}
\newtheorem{definition}[theorem]{Definition}
\newtheorem{assumption}[theorem]{Assumption}
\theoremstyle{remark}
\newtheorem{remark}[theorem]{Remark}

\usepackage[textsize=tiny]{todonotes}

\title{Taking a Moment for Distributional Robustness}
\author[1]{Jabari Hastings} 
\author[1]{Christopher Jung}
\author[1]{Charlotte Peale}
\author[2]{Vasilis Syrgkanis}
\affil[1]{Department of Computer Science, Stanford University}
\affil[2]{Department of Management Science and Engineering, Stanford University}

\begin{document}

\maketitle
\begin{abstract}

    A rich line of recent work has studied \emph{distributionally robust} learning approaches that seek to learn a hypothesis that performs well, in the worst-case, on many different distributions over a population. We argue that although the most common approaches seek to minimize the worst-case \emph{loss} over distributions, a more reasonable goal is to minimize the worst-case \emph{distance} to the true conditional expectation of labels given each covariate. Focusing on the minmax loss objective can dramatically fail to output a solution minimizing the distance to the true conditional expectation when certain distributions contain high levels of label noise. We introduce a new min-max objective based on what is known as the \emph{adversarial moment violation} and show that minimizing this objective is equivalent to minimizing the worst-case $\ell_2$-distance to the true conditional expectation if we take the adversary's strategy space to be sufficiently rich. Previous work has suggested minimizing the maximum \emph{regret} over the worst-case distribution as a way to circumvent issues arising from differential noise levels. We show that in the case of square loss, minimizing the worst-case regret is also equivalent to minimizing the worst-case $\ell_2$-distance to the true conditional expectation. Although their objective and our objective both minimize the worst-case distance to the true conditional expectation, we show that our approach provides large empirical savings in computational cost in terms of the number of groups, while providing the same noise-oblivious worst-distribution guarantee as the minimax regret approach, thus making positive progress on an open question posed by \citet{agarwal2022minimax}.

\end{abstract}

\section{Introduction}
While standard machine learning approaches typically focus on achieving a low average loss over the entire training set, a rich line of work has highlighted how optimizing for accuracy on average can result in models that perform badly on small or underrepresented subpopulations in the data 
\citep{buolamwini2018gender, hebert2018multicalibration, kearns2018preventing}.

This issue has motivated the study of \emph{distributionally robust} approaches to learning that aim to achieve good worst-group performance on a set of distributions over the overall population space. While distributionally robust optimization (DRO) tasks take on a variety of forms, most focus on minimizing the worst-distribution loss over a set of distributions. More formally, given a hypothesis space $\cH$, a set of distributions $\cD$ over a covariate and output space $\cX \times \cY$, and some loss function $\ell: \cY \times \cY \rightarrow \R$, we refer to a \emph{minmax loss} approach to DRO as seeking to find the solution to 
\[\argmin_{h \in \cH}\max_{D \in \cD} \bE_{(X, Y) \sim D}[\ell(h(X), Y)].\]

In the non-distributionally-robust setting where the goal is to learn a good model over a single population distribution, loss-minimization objectives act as the gold standard learning approach due to their inherent connections to learning the true conditional expectation of $Y$ given $X$, i.e., $h_0(X) := \mathbb{E}_D[Y|X]$. When we only have access to data points $(x, y)$, with noisy labels $y$, it is difficult to evaluate the performance of a learned model with respect to how far it is from the conditional expectation $h_0$. However, minimizing a properly chosen loss function over the data can act as a convenient proxy for minimizing the learned model's distance to $h_0$. As one example, the standard objective of square loss, i.e., $\ell(h(X), Y) = (h(X) - Y)^2$, 
 is desirable because given a distribution $D$, minimizing the square-loss can be re-expressed as 
 \begin{align*}
     \argmin_h \mathbb{E}_D[(h(X) - Y)^2]= \argmin_h \mathbb{E}_D[(h(X) - h_0(X))^2] + \mathbb{E}_{X \sim D_X}[\mathsf{var}_D(Y|X)] 
 \end{align*}
 where the additional variance term $\mathbb{E}_{X \sim D_X}[\mathsf{var}_D(Y|X)]$ is invariant in $h$. Thus, minimizing $\mathbb{E}_D[(Y - h(X))^2]$ is equivalent to minimizing $\mathbb{E}_D[(h_0(X) - h(X))^2]$, allowing the learner to optimize for $\ell_2$-distance to $h_0$ without knowledge of $h_0$ itself.

At first glance, the minmax loss problem may seem like a natural extension of this loss-minimization objective to the multi-distribution setting. However, this equivalence between minimizing a loss over the data labels and minimizing the distance from $h_0$ begins to break down when we move to cases with multiple distributions. Returning to the example of square loss in cases with differential levels of label noise, we note that while the quantity $\mathbb{E}_{X \sim D_X}[\mathsf{var}_D(Y|X)]$ is invariant to $h$, it is far from invariant to the distribution $D$ and may be very different for different distributions. Thus, optimizing for the worst-distribution loss over a set where certain distributions have high levels of label noise may cause the optimization to prioritize hypotheses that perform well on these high-noise distributions even if there are other hypotheses that are much closer to $h_0$ on all distributions simultaneously --- see \citet{agarwal2022minimax} for more concrete examples of .

In other words, if we believe that the role of loss minimization (e.g. square loss) is to facilitate learning a model $h$ whose distance to the conditional expectation $h_0$ is close, it follows that the worst case group performance in the multi-distribution setting should be also measured in terms of model $h$'s worst case distance to $h_0$ over different distributions, not necessarily $h$'s worst case loss over the distributions. In the case of squared loss as an example, we would ideally want to find $h$ that minimizes $\max_{D \in \cD} \bE[(h(X) - h_0(X))^2]$. Such a goal requires new techniques that go beyond standard worst-case-loss minimization.

Motivated by developing new techniques for minimizing the worst-distribution distance to $h_0$, we focus on the goal of minimizing the worst-case $\ell_2$ distance to $h_0$, $\max_{D \in \cD} \bE_{D}[(h(X) - h_0(X))^2]$. To achieve this goal, we propose an objective based on ``adversarial moment violation'' \cite{dikkala2020minimax}. The adversarial moment violation for hypothesis $h$ with respect to distribution $D$ and some adversary class $\cF \subseteq \cY^{\cX}$ is defined as 
\begin{align}\label{eqn:moment-violation}
\max_{f \in \cF}\bE_{(X,Y) \sim D}[(Y-h(X))f(X)].
\end{align}
This is also known as multi-accuracy error in the algorithmic fairness literature \cite{kim2019multiaccuracy} where $\cF$ represents a class of key subpopulations. Minimizing this objective guarantees that there is no way to reweigh the population via $f \in \cF$ such that the bias of $h(x)$ with respect to $\bE[Y|X=x]$ is too large. 

We show that minimizing across $D \in \cD$ the worst adversarial moment violation \eqref{eqn:moment-violation} with some regularization penalties for some sufficiently rich adversary class $\cF$ is equivalent to minimizing the worst-case $\ell_2$ distance to $h_0$ relative to the distance that could have been achieved by focusing only on $D$. More formally, optimizing the worst adversarial moment violation is equivalent to minimizing the following quantity:
\begin{align}\label{eqn:agnostic-worst-case-distance}
    \max_{D \in \cD} \bE_D[(h(X) - h_0(X))^2] - \bE_{D}[(h_D(X) - h_0(X))^2]
\end{align}
where $h_D = \arg\min_{h \in \cH}\bE_{(X,Y) \sim D}[(h(X) - h_0(X))^2]$ is the best hypothesis for each specific distribution $D \in \cD$. In the realizable case where $h_0 \in \cH$ and hence $\bE_{D}[(h_D(X) - h_0(X))^2] = 0$ for each $D \in \cD$, optimizing the above objective \eqref{eqn:agnostic-worst-case-distance} is equivalent to minimizing $\max_{D \in \cD} \bE_{D}[(h(X) - h_0(X))^2]$, so one may think of the above objective as an ``agnostic'' version of minimizing the worst case distance to $h_0$, and we will now refer to the objective in \eqref{eqn:agnostic-worst-case-distance} as the worst-case agnostic $\ell_2$ distance to $h_0$ with respect to $\cH$. Furthermore, we show that the worst-case adversarial moment violation objective \eqref{eqn:moment-violation} can also be tuned to provide a trade-off between minimizing the worst-case agnostic $\ell_2$ distance and multi-accuracy error over $\cF$, $\max_{f \in \cF, D \in \cD}\bE[(h_0(X) - h(X))f(x)]$.

Most related to work to ours is \citet{agarwal2022minimax}. While the authors do not explicitly deal with the goal of minimizing the distance between $h_0$ and $h$, they point out some of the failures of the minmax loss approach in settings with differential levels of label noise through the lens of \emph{regret minimization}. In order to combat over-fitting $h$ to inherently hard-to-learn distributions $D$ (e.g. due to high levels of noise, mismatch between $H$ and $h_0$, etc), they take the inherent difficulty of learning over each distribution $D \in \cD$ into account by optimizing for minmax \emph{regret} as opposed to minmax loss:
\begin{align*}
    \min_{h \in \cH}\max_{D \in \cD} 
    \bE_{(X, Y) \sim D}[\ell(h(X), Y)] - \bE_{(X,Y) \sim D}[\ell(h_D(X), Y)]
\end{align*} 
where $h_D = \arg\min_{h \in \cH}\bE_{(X,Y) \sim D}[\ell(h(X), Y)]$ is the best hypothesis for each specific distribution $D$.

Note that the above minmax regret objective of \citet{agarwal2022minimax} in the case of square loss is equivalent to worst-case agnostic $\ell_2$ distance to $h_0$: 
\begin{align*}
    &\min_{h \in \cH}\max_{D \in \cD} 
    \bE_{(X, Y) \sim D}[(h(X)- Y)^2] - \bE_{(X,Y) \sim D}[(h_D(X) - Y)^2]\\
    &=\min_{h \in \cH}\max_{D \in \cD} 
    \bE_{(X, Y) \sim D}[(h(X)- h_0(X))^2] - \bE_{(X,Y) \sim D}[(h_D(X) - h_0(X))^2].
\end{align*}
In other words, minimizing the worst-case agnostic $\ell_2$ distance to $h_0$ can be achieved by optimizing our proposed objective based on adversarial moment violation or their worst-case regret with squared loss. However, we argue that our approach has run-time advantages compared to theirs.

Their computational approach, Minmax Regret Optimization (MRO), requires a simple modification to the no-regret learning approach considered in \citet{sagawa2019distributionally}: they calculate the minimum empirical loss for each $D \in \cD$ by running an empirical risk minimization (ERM) on each $D \in \cD$ and center the minmax loss objective around this empirical risk. Hence, the run-time of their algorithm scales linearly with the total number of distributions $|\cD|$.

Theoretically, our approach still needs $|\cD|$ adversaries to be concurrently trained at each step, meaning that naively our new method doesn't seem to improve much over MRO. However, we show that in a practical case where the hypothesis space is a neural network, our objective can be simplified in a way that reduces costly adversarial training steps, providing significant empirical savings in computational cost. We highlight this as positive progress toward an open question raised by \cite{agarwal2022minimax}, who asked whether there exist approaches to minmax regret optimization with better dependence on the number of groups and/or can scale better to larger datasets. 

We highlight our method as a new approach for distributionally robust optimization that adapts to unknown heterogeneous noise levels across distributions. Our method works especially well in regimes where we want to optimize over a large number of distributions, as it improves computation time compared to existing methods that rely on computing an ERM over each distribution. 

\subsection*{Our Contributions}
\begin{enumerate}
    \item Theoretical analysis showing equivalence between minimizing the worst adversarial moment violation and the worst-case agnostic $\ell_2$ distance to the conditional expectation $h_0$. (Section~\ref{sec:moment-thms}) 
    \item Finite sample analysis for our worst adversarial moment violation objective for common hypothesis spaces (Section~\ref{sec:finite-sample})
    \item Computationally efficient algorithms for minimizing the worst adversarial moment violation (Section~\ref{sec:comp})
    \item Experiments on synthetic datasets with RKHS and real world data set (CelebA) with neural networks demonstrating improved runtime and comparable performance (Section~\ref{sec:exp})
\end{enumerate}

\subsection{Related Work}
\paragraph{Distributionally Robust Optimization} As already mentioned, \emph{distributionally robust optimization} refers to a general class of learning tasks where given a set of distributions the goal is to output a model that maximizes the performance of the worst distribution. Prior work has investigated methods for optimizing for a variety of types of distribution sets defined by various measures of distance between distributions; see \citet{delage2010distributionally, duchi2021learning, hu2013kullback, erdougan2006ambiguous, shafieezadeh2015distributionally}. \citet{koh2021wilds} have curated 10 datasets that incorporate a wide range of distribution shifts that occur naturally in the real world scenarios. 

Our work is particularly interested in the setting where the distribution set contains many subpopulations present in the data. One of the most influential works that study this learning task is the groupDRO approach of~\citet{sagawa2019distributionally}, where they seek a solution that minimizes the worst-group loss, and term this approach groupDRO. Follow-up works including \citet{chen2017robust} and \citet{NEURIPS2022_02917ace} have offered further theoretical analyses of the groupDRO approach, while others such as \citet{liu2021just} have suggested simple algorithms that provide performance comparable with that of groupDRO but without access to the group membership. 

Our paper is not the first to observe that groupDRO's minmax loss approach may fail in certain settings. As discussed above, \citet{agarwal2022minimax} noted that groupDRO is sensitive to heterogeneous noise and proposed minimizing minmax regret. \citet{lei2023policy} study minmax regret in the context of treatment effect estimation. \citet{zhai2021doro} observe that groupDRO performs quite poorly on many large-scale tasks possibly as a result of instability and shows that this instability may be caused by groupDRO's sensitivity to outliers. They propose new methods that are robust to outliers.

\textbf{Method of Moments }
Starting with \citet{hansen1982large}, there has been an extensive line of work on the method of moments. It originally started as a way of estimating the population parameters by matching the empirical moments of a distribution to those of the population. Although method of moments shares a similar goal as machine learning in trying to estimate some unknown distribution, there has been little work in marrying these two ideas together until \cite{lewis2018adversarial} and subsequently \cite{bennett2019deep,dikkala2020minimax,kaji2023adversarial}. More specifically, it is hard to develop statistical learning theory for method of moments when the data is finite and/or the models are non-parametric because imposing infinitely many conditional moment restrictions with finite data is impossible. Therefore, \cite{lewis2018adversarial,dikkala2020minimax} formulate a game between a learner that is trying to satisfy all conditional moment restrictions and an adversary trying to identify a test function that witnesses the violation of a conditional moment criterion. By formulating as a zero-sum game between a learner and an adversary like Generative Adversarial Networks (GANS) \cite{goodfellow2014generative}, they extend the method of moments to settings where data is finite and/or models are non-parametric. 

\textbf{Multi-accuracy \& Multi-calibration \& Multi-group Fairness}
While developed separately, the adversarial method of moments objective bears similarities to the notions of multiaccuracy and multicalibration first studied by \citet{kim2019multiaccuracy}
and \citet{hebert2018multicalibration} in an algorithmic fairness context. While most work in this area has been purely theoretical, \citet{kim2019multiaccuracy} study efficient ways to post-process a predictor to get multicalibration guarantees and show that this approach improved accuracy on minority subgroups in a variety of applications. 

\citet{DianaGGK0S21} also propose a multi-group fairness definition called \emph{lexicographic minimax fairness}. This definition draws connections between the notion of a \emph{lexicographically maximal} solution in allocation problems and first asks that the error of the worst case group be minimized, and among all the solutions that minimize the worst case group's error, the definition asks for a solution that minimizes the error of the second worst case group, iterating until all groups have been maximized subject to maximizing the utility of the worse-off groups. 

\textbf{Out-of-Domain Generalization} 
Similar to distributionally robust optimization, invariant risk minimization (IRM), first introduced by \citet{arjovsky2019invariant}, seeks a similar goal of trying to be robust to distribution shifts. They essentially assume that there exists a canonical predictor that must be good across all domains because it is not relying on spurious correlations with domain specific features and seek to learn such predictor. \citet{rosenfeld2020risks} show that the sample complexity for IRM in terms of the number of domains in order to generalize to unseen environment is quite high; they argue that in many settings, IRM objective and many variants can be expected to perform no better than ERM. \citet{piratla2022focus} observe that groupDRO performs much worse than ERM on some data sets, instead propose an approach that favors learning features that are useful for all groups. 

\section{Fair Solutions via Worst-Case Moment Violation Minimization}\label{sec:moment-thms}
In this paper, we present an alternative objective to distributionally robust learning where we aim to learn $h$ by ensuring that it minimizes the maximum violation of a set of moment conditions over the potential distributions. More precisely, for a set of functions $\cF \subseteq \cY^{\cX}$, we seek to find $h$ that minimizes the following objective:
\begin{align*}
\min_{h \in \cH}\max_{D \in \cD}\max_{f \in \cF}\bE_{(X, Y) \sim D}[2\,(Y - h(X))\,f(X) - f(X)^2]
\end{align*}
We refer to this approach as the \emph{minmax adversarial moments} approach. We relate this goal to optimizing the worst-case agnostic $\ell_2$ distance to the true conditional expectation of $Y$ given $X$, $h_0(X) := \mathbb{E}_D[Y|X]$. For simplicity, we assume that $h_0(X)$ is the same for all distributions in $\cD$, but our results extend to the case where $h_0$ is distribution-specific for each $D \in \cD$. We will show that this approach translates to minimizing \eqref{eqn:agnostic-worst-case-distance}, the worst-case $\ell_2$ distance to $h_0$ relative to $h_D = \arg\min_{h \in \cH}\bE_{(X,Y) \sim D}[(h(X) - h_0(X))^2]$.

Note that throughout the paper, we make the minimal assumption that for every distribution $D \in \cD$, there exists some hypothesis in $\cH$ that performs well on $D$ (but not necessarily on other distributions):

\begin{assumption}[$\epsilon$-Distribution Specific Approximation Error]\label{ass:main} For any distribution $D\in \cD$, there exists a hypothesis $h_D\in \cH$, such that $\|h_D-\hz\|_{L^2(D)}^2\leq \epsilon$.
\end{assumption}

As a first step in our analysis, we show the following crucial lemma that gives an alternative characterization of the maximum violation criterion. In fact, we characterize a slightly more general form of the maximum violation criterion that will also be useful in the finite sample analysis.
\begin{lemma}[Completing the Square]\label{lem:square} Fix any number $c \neq  0$. Consider the maximum violation criterion:
\begin{align*}
    V_D(h) := \max_{f \in \cF}\bE_{(X, Y) \sim D}[2\,(Y - h(X))\,f(X) - c\, f(X)^2].
\end{align*}
Then we have:
\begin{align*}
    V_D(h) = \frac{1}{c} \|h_0 - h\|_{L^2(D)}^2 - \frac{1}{c}\min_{f\in \cF} \|h_0 - h - c\,f\|_{L^2(D)}^2
\end{align*}
\end{lemma}
\begin{proof}
Consider the shorthand notation $\bE_{D}$ for $\bE_{(X,Y)\sim D}$. First we note that for any $D$ the slightly modified adversarial objective satisfies, by the tower law of expectation:
\begin{align*}
    V_D(h) =~&\max_{f\in F} \bE_D\left[2\,(h_0(X) - h(X))\, f(X) - c\,f(X)^2\right]
\end{align*}
Let $\nu=h_0 - h$. Then we re-write the adversarial objective for any $D$ by completing the square:
\begin{align*}
    &\max_{f \in \cF} \bE_{D}\left[2\nu(X)f(X) - c f(X)^2\right] \nonumber\\
    =~& \max_{f \in \cF} \bE_{D}\left[\frac{1}{c}\nu(X)^2 - \frac{1}{c}\nu(X)^2 +  2\nu(X)f(X) - cf(X)^2\right]\nonumber\\
    =~& \max_{f \in \cF} \bE_{D}\left[\frac{1}{c}\nu(X)^2 - \frac{1}{c} (\nu(X)- c\,f(X))^2\right]\nonumber\\
    =~& \frac{1}{c} \|h_0 - h\|_{L^2(D)}^2 - \frac{1}{c}\min_{f\in \cF} \|h_0 - h - c\,f\|_{L^2(D)}^2 \qedhere
\end{align*} 
\end{proof}

As a corollary of this statement, we derive the following sandwich bounds that we use throughout our proofs. 

\begin{corollary}[Upper Bound]\label{cor:sandwich-ub}
For any $c > 0$, we have 
\begin{align*}
    \max_{D\in \cD, f \in \cF} \bE_{D}\left[2(Y - h(X))f(X) - c\,f(X)^2\right] \leq \frac{1}{c} \max_{D\in \cD} \|h_0 - h\|_{L^2(D)}^2.
\end{align*}
\end{corollary}

\begin{corollary}[Lower Bound]\label{cor:sandwich-lb}
Let $h_D$ denote the best approximation to $h_0$ with respect to distribution $D \in \cD$ as designated by the Assumption~\ref{ass:main} for which $\|h_0-h_D\|_{L^2(D)}^2\leq \epsilon$. For any $c > 0$ such that $\frac{1}{c}(h_D - h)\in \cF$, we have
\begin{align*}
    \frac{1}{c} \left(\max_{D\in \cD}\|h_0 - h\|_{L^2(D)}^2 - \epsilon\right) \leq \max_{D\in \cD, f \in \cF} \bE_{D}\left[2(Y - h(X))f(X) - c\,f(X)^2\right].
\end{align*} 
\end{corollary}

\begin{remark}
    We note that this corollary is tight when we take $c = 1$ and set a separate $\epsilon_D = \|h_0-h_D\|_{L^2(D)}^2$ for each $D \in \cD$. For this setting of parameters, we get the following equivalence between the worst-case agnostic $\ell_2$ distance to $h_0$ and our adversarial moments objective:
\begin{align*}
    \max_{D\in \cD}\|h_0 - h\|_{L^2(D)}^2 - \|h_0-h_D\|^2_{L^2(D)} = \max_{D\in \cD, f \in \cF} \bE_{D}\left[2(Y - h(X))f(X) - c\,f(X)^2\right].
\end{align*}

Therefore, all the algorithms presented later in Section~\ref{sec:comp} in fact ouput an approximately optimal square loss minmax regret solution. 
\end{remark}

We now present our main theorem:
\begin{theorem}[Population Limit, MSE]\label{thm:population-main-mse} Let $h_*$ be a solution to the min-max adversarial moment problem:
\begin{align*}
    \min_{h\in \cH} \max_{D \in \cD}\max_{f \in \cF}\bE_{(X, Y) \sim D}\left[2\,(Y - h(X))\,f(X) - cf(X)^2\right]
\end{align*}
with $c > 0$, $\cF \supseteq \shull(\cH-\cH) := \{\gamma (h-h'): h,h'\in \cH, \gamma\in [0,1]\}$. If the pair $\cH, \cD$ satisfies Assumption~\ref{ass:main}, then for any $c > 0$,
\begin{align*}
    \max_{D\in \cD} \|h_*-\hz\|_{L^2(D)}^2\leq \max\left\{\frac{1}{c}\wedge 1\right\}\left(\min_{h \in H}\max_{D \in \cD} \|h-\hz\|_{L^2(D)}^2\right) + \epsilon.
\end{align*}
\end{theorem}

For applications where the main goal is minimizing the $\ell_2$ distance to $h_0$, choosing $c = 1$ gives the optimal approximation guarantee. However, we state our theorem in this more general manner because smaller values of $c$ can provide a trade-off between the mean-square error objective and the multigroup fairness measure of \emph{multiaccuracy}, mentioned earlier. We define it formally as follows:

\begin{definition}[Multiaccuracy Error]
Given a distribution $\cD$ over $\cX \times \cY$, and a class of auditing functions over $\cF \subseteq \cY^X$, we define the multiaccuracy error of h on $\cD$ with respect to $\cF$ as:
\begin{align*}
    \textsf{MA-err}_\cD(h, \cF) := \sup_{f \in \cF}\left|\bE_{\cD}[(Y - h(X))f(X)]\right|.
\end{align*}
\end{definition}

Intuitively, the multiaccuracy error metric can be thought of as measuring the bias of $h$ over a large set of test functions representing subsets of the population, rather than just overall. We observe that when $\cF$ is closed under negation, setting $c = 0$ in our objective is exactly equivalent to minimizing the worst-distribution multiaccuracy error. 

For larger values of $c$, we can reinterpret the adversarial objective as measuring an ``$\ell_2$-penalized'' version of multiaccuracy error where the additional $\bE_{D}[-cf(x)^2]$ term penalizes multiaccuracy violations that use test functions with a large $\ell_2$-norm. 

Thus, for those interested in both multiaccuracy and mean-square error guarantees, an appropriate $c$ can be chosen to produce solutions that trade off between these two objectives by increasing or decreasing the $\ell_2$-penalty. See Appendix~\ref{app:ma-guarantees} for more details.

\section{Finite Sample Analysis}\label{sec:finite-sample}
We provide finite sample guarantees for sample analogues of our proposed criterion, exchanging the multiplier $c$ with a dataset-size dependent $a_n$ to capture settings where the algorithm designer may wish to tune the multiplier in a size-dependent way. Let $\cD:=\{D_1, \ldots, D_M\}$ and let $\{S_1, \ldots, S_M\}$ be a collection of datasets, with $S_j$ containing $n$ samples $(X_i, Y_i)$ drawn i.i.d. from distribution $D_j\in \cD$. For each population, we can replace the population quantity $\bE_{D_j}[2(Y-h(X))\, f(X)-cf(X)^2]$ by its empirical analogue:
\begin{align*}
    &\bE_{j, n}[2(Y - h(X))f(X) - a_nf(x)^2] := \frac{1}{n} \sum_{(X_i, Y_i)\in S_j} 2\,(Y_i - h(X_i))f(X_i) - a_nf(X_i)^2
\end{align*}

To analyze the finite sample behavior of our minmax adversarial moment approach, we will introduce a measure of statistical complexity that most times captures tight convergence rates for strongly convex risk functions and has been well-studied in the statistical learning theory literature.

For any function space $\cF$, containing functions that are uniformly and absolutely bounded by $U$, we will be using the critical radius as the measure of statistical complexity (cf.\ \citet{wainwright2019high} for a more detailed exposition). To define the critical radius, we first define the localized Rademacher complexity:
\begin{align*}
    \cR(\cF, \delta) = \frac1{2^n}\sum_{\epsilon\in\{-1,1\}^n}\bE\left[\sup_{f\in \cF: \|f\|_{L^2}\leq \delta} \frac{1}{n}\sum_{i=1}^n \epsilon_i f(X_i)\right].
\end{align*}
W
The star hull of a function space is defined as $\shull(\cF)=\{\gamma f: f\in \cF, \gamma\in [0,1]\}$. The critical radius $\delta_n$ of $\cF$
is the smallest positive solution to the inequality:
\begin{align*}
    \cR(\shull(\cF), \delta) \leq \delta^2/U
\end{align*}

We note that we analyze the star hull of function classes not only because it allows us to use the ``completing the square'' lemma (Lemma~\ref{lem:square}) but also because the critical radius theory makes use of the star-convexity of function classes when applying Talagrand's inequality. For VC-subgraph classes, with VC-dimension $d$, the critical radius can be shown to be upper bounded by $O\left(\sqrt{\frac{d\log(n)}{n}}\right)$. For RKHS function spaces, with the Gaussian kernel, it has also been shown to be $O\left(\sqrt{\frac{\log(n)}{n}}\right)$. The critical radius has been analyzed for many function spaces used in the practice of machine learning \cite{wainwright2019high}.

Using the complexity measure above, we can show a generalization result (Theorem~\ref{thm:interpolate-mse} in Appendix~\ref{app:finite-sample-details}) that says the empirical worst-case adversarial moment violation of any $h \in \cH$ evaluated on finite sample must be close to the true quantity and hence $\hat{h}$ that is empirically optimal in terms of adversarial moment violation must have small $\ell_2$ distance to $h_0$ via Corollary \ref{cor:sandwich-ub} and \ref{cor:sandwich-lb}. 

We further improve this result by adding appropriate regularization terms on $f$ and $h$ to our worst-case adversarial moment objective. Regularizing the adversarial moment violation with some norm on $f$ and $h$ allows us adapt to the norm of $h_D$ (i.e. the smaller the $||h_D||$ the better generalization result we can guarantee) and allow us to weaken some of our assumptions on the function classes, which we discuss more formally later in Remark~\ref{rem:regularization-finite} after we state the our regularized result below. 

\begin{theorem}[MSE guarantee, Finite Samples Regularized]\label{thm:interpolated-reg-mse} Let $\hat{h}$ be the solution to the empirical regularized min-max adversarial moment problem:
\begin{align*}
    \min_{h\in \cH} \max_{j=1}^M
    L_{j,n}(h) + \mu \|h\|^2
\end{align*}
where 
\begin{align*}
    &L_{j,n}(h) := \max_{f\in F} \bE_{j,n}\left[2(Y - h(X))\, f(X) - a_n f(X)^2\right] - \lambda \|f\|^2
\end{align*}
for any %
$0 < a_n \leq 1/2$ with $\cF=\shull(\cH-\cH)$ and $\|\cdot\|$ an arbitrary norm in the space $\cH-\cH$. Let $h_j = \argmin_{h\in \cH} \|h-h_0\|_{L^2(D_j)}^2$ and $F_B=\{f\in F:\|f\|\leq B\}$ for some constant $B$. Let $\delta_n=\Omega\left(\sqrt{\frac{\log\log(n) + \log(M/\zeta)}{n}}\right)$ be any solution to the critical radius inequality for $F_B$. Let $h_*=\argmin_{h\in \cH} \max_{D\in \cD} \|h-h_0\|_{L^2(D)}^2$. Assume that $Y - h_j(X)$ and $Y-h_*(X)$ are a.s. absolutely bounded by $U_Y$ and all $f\in \cF_B$ are a.s. absolutely bounded by $U$. Assume that the pair $\cH, \cD$ satisfies Assumption~\ref{ass:main}. Then for $\lambda \geq 4 c_0^2 (U_Y \wedge U^2)\delta_n^2/B^2$ and $\mu\geq 4\left(\lambda+c_0^2\delta_n^2 \frac{U^2}{B^2}\right)$, we get that w.p. $1-\zeta$:
\begin{align*}
    \max_{D\in \cD} \|h_0 - \hat{h}\|_{L^2(D)}^2 \leq \frac{8}{a_n}\min_{h\in \cH} \max_{D\in \cD} 
     \|h_0 - h\|_{L^2(D)}^2 + \frac{10\epsilon}{a_n} + 8\mu \max_j\{\|h_j\|^2 \wedge \|h_*\|^2\wedge 1\} + O\left(\delta_n^2 (U_Y^2\wedge U^2)/a_n\right)
\end{align*}
where $c_0 \leq 18$ is a sufficiently large universal constant.
\end{theorem}

We defer the proof as well as alternative versions of this theorem, including the unregularized version, to Appendix~\ref{app:finite-sample-details}.
\begin{remark}[Advantages of Regularization]\label{rem:regularization-finite}
The regularized version (Theorem~\ref{thm:interpolated-reg-mse}) above requires an upper-bound on the critical radius of $\cF_B$ whereas the unregularized version (Theorem~\ref{thm:interpolate-mse}) requires an upper-bound on the critical radius of the entire $\cF$, not $\cF_B$, and an upper-bound on the critical radius of the cross-product of $\cF$ and $\cH$. Moreover, the regularized version requires an absolute upper bound $U$ only on each $f \in \cF_B$, but the unregularized version requires one on each $f \in \cF$. Finally, note how the regularized version adapts to the norm of the $h_j$ and $h_*$.
\end{remark}
\begin{remark}
    In the above case, choosing $a_n=\frac{1}{2}$ yields the best result in terms of square loss. A smaller $a_n$ that decays with $n$ implies better multi-accuracy guarantees, which we show in Appendix~\ref{app:ma-guarantees}. Hence, $a_n$ serves as a tradeoff between these two criteria. 
\end{remark}

\subsection{Sample Complexity for Linear Function Classes} 
We consider the special case where the hypothesis space consists of norm-constrained linear functions for a known feature map $\phi(\cdot)\in \R^d$, i.e., $\cH = \{\alpha^\top\phi(\cdot): \alpha \in \R^d, \|\alpha\|_2 \leq A \}$.
By known results in localized complexities \cite{wainwright2019high}, 
the critical radius $\delta_n$ of a linear function class  whose functions are uniformly and absolutely bounded by $U$ can be upper bounded by any solution to the inequality $c_0 \delta \sqrt{d/n} \leq \delta^2 / U$, which can be taken to be of the order 
$\delta_n= O \left ( \max \left \{ \sqrt{\tfrac{\log\log(n) + \log(M/\zeta)}{n}},  U\sqrt{\tfrac{d}{n}} \right \} \right )$, when $\phi(\cdot)\in \R^d$. If we also assume that $\|\phi(X)\|_2\leq V$, a.s., then we also know that the $\ell_2$ norm constrained function space, with $\|\alpha\|_2\leq 1/V$ contains functions that are uniformly and absolutely bounded by $1$. 

Assuming that $|Y-\alpha_j^\top\phi(X)|\leq U_Y$ and $|Y-\alpha_*^\top\phi(X)|\leq U_Y$ (for some $U_Y\geq 1$), we can apply Theorem~\ref{thm:interpolated-reg-mse}, with $B=1/V$ and $U=1$ and $\lambda = 4 c_0^2\delta_n^2 U_Y V^2$ and $\mu= 4\left(\lambda + c_0^2\delta_n^2 V^2\right)$, we can get the statistical rate: with probability $1-\zeta$,
\begin{align*}
    \max_{D\in \cD} \|h_0 - \hat{h}\|_{L^2(D)}^2 
     \leq 16\min_{h\in \cH} \max_{D\in \cD} \|h_0 - h\|_{L^2(D)}^2 + 20\epsilon 
     + O\left(\frac{\tilde{d}}{n} U_Y V^2 \max_{j=1}^M\{\|\alpha_j\|_2^2 \wedge \|\alpha_*\|_2^2\wedge 1\} + \frac{\tilde{d}}{n} U_Y^2\right)
\end{align*}
where $\tilde{d} = \max \{ d, \log \log (n) + \log (M / \zeta) \}$.

Similar results can be obtained for the norm-constrained un-regularized variant, i.e. where $\cH_B=\{\alpha^\top\phi(\cdot): \alpha \in \R^d, \|\alpha\|_2\leq B\}$, for some constant $B$ (see Appendix~\ref{app:lin-samp-cxty}). 

\section{Efficient Computation} \label{sec:comp}
We also show how to compute approximate solutions to the empirical minmax adversarial moment problem for a number of practical choices of hypothesis spaces. For simplicity, we set the parameter $a_n = 1$, so that the moment violation criterion is
\begin{align*}
    L_{j,n}(h, f) = \bE_{j, n}\left [ 2 (Y - h(X))f(X) - f(X)^2 \right ].
\end{align*}
We obtain approximate solutions by invoking the classic result of \cite{freund1999adaptive}, which frames the problem as a zero-sum game between two players respectively attempting to minimize and maximize the moment violation criterion $L_{j, n}(h, f)$. In the settings we consider, each player can deploy \textit{no-regret} algorithms to obtain a sequence of solutions whose difference from the optimum is bounded by some term that converges to zero as the number of rounds in the game increases. 

In the most general case, when $\cH$ is a convex hypothesis space, our oracle-efficient algorithm must employ a separate adversary for each distribution, meaning that $O(M)$ oracle calls must be made during each iteration, and thus not significantly improving over the MRO approach, which makes $O(M)$ initial oracle calls to compute an ERM over each distribution and then $O(M)$ calls per iteration.

However, the benefits of our approach are highlighted in two natural special cases of the more general convex setting: linear spaces and their generalization, reproducing kernel Hilbert spaces. In this special case, we show that our adversarial moments objective simplifies to a closed-form solution for the optimal adversarial weights, eliminating the need for $M$ separate adversaries that must be updated at each iteration (see Lemma~\ref{lem:rewrite-linear}). This means that during each iteration, we only need to compute a single costly matrix inverse operation to calculate the learner's best response. 

Our theoretical results still show that the runtime of each iteration is at least $O(M)$ due to additional matrix multiplication operations that are necessary, however we are able to show empirical computational savings for this approach because it reduces the number of costly matrix inverse operations to $O(1)$ per-iteration. We empirically demonstrate the computational savings of our approach for the case of reproducing kernel Hilbert spaces in Section~\ref{subsec:exp-rkhs}, and additionally show through experiments that a similar closed-form property can be applied in the case of neural networks by extending the adversarial model's feature vector to contain group-membership indicators and allowing us to simultaneously train the $M$ adversaries with a single model (Section~\ref{subsec:exp-neural}).

\subsection{Convex Spaces}
We begin with our most general result for convex hypothesis spaces. We consider the case that $\cH$ is a convex hypothesis space and $\cF = \shull(\cH - \cH)$. Note that $\shull(\cH - \cH)$ is also a convex hypothesis space. 

For this general setting, we obtain oracle-efficient algorithms for obtaining approximate solutions to the unregularized version of the problem (see Appendix \ref{app:convex-spaces} for more details).

\begin{theorem}[Approximate Solution for Convex Spaces]
\label{thm:approx-sol-convex}
Let $\cH$ be a compact, convex hypothesis space and suppose that $\cF = \shull(\cH - \cH)$. 
There is an oracle-efficient algorithm  that runs for $T$ iterations and returns a $\gamma_T$-additive approximate solution to the empirical min-max adversarial moment problem, 
where 
$
\gamma_T = O ( \log (T)/T +  \sqrt{ \log(M) / T }  )
$ and $O(M)$ oracle calls are made during each iteration. 

\end{theorem}

\subsection{Linear Spaces}\label{subsec:linear}

We now show that when $\cH$ is a linear hypothesis space, we can significantly simplify the adversarial moments objective to get a closed form-solution for the adversary's best response without the need for additional oracle calls to each adversary.

We consider the case where the hypothesis space consists of norm-constrained linear functions for a known feature map $\phi(\cdot)\in \R^d$, i.e., $\cH = \{\alpha^\top\phi(\cdot): \alpha \in \R^d, \|\alpha\|_2 \leq A \}$. Note that $\cH$ is a compact space. Suppose that we also have $\cF = \cH$. Since $\cH-\cH$ is the same space of linear functions, we have $\cF\supseteq \shull(\cH-\cH)$.

We observe that the empirical regularized min-max adversarial moment problem can be rewritten in a form that can be approximately solved via no-regret dynamics.
\begin{lemma}[Adversarial Moment Problem, Restated for Linear Spaces]\label{lem:rewrite-linear}
Let $X_{ij}, Y_{ij}$ denote the $i$-th sample from the $j$-th dataset. Furthermore, let  $y_j = [Y_{1j}; \dots; Y_{nj}] \in \mathbb R^{n}$ and $\Phi_j = [\phi(X_{1j})^\top; \dots; \phi(X_{nj})^\top] \in \mathbb R^{n \times d}$. Then
\begin{align*}
\nonumber
 \min_{h \in \cH} \max_{j \in [M]} \max_{f \in \cF} L_{j,n}(h, f) - \lambda \|f\|^2 + \mu \|h\|^2  =\min_{\substack{\alpha \in \mathbb R^d \\ \|\alpha\|_2 \leq A}} \max_{w \in \Delta(M)}  \frac{1}{n} \sum_{j=1}^M w_j \left (\kappa_j - 2 \nu_j^\top \alpha + \alpha^\top \Sigma_j \alpha \right ) 
\end{align*}
where
\begin{align*}
\kappa_j :=~& y_j^\top Q_j y_j  &
\nu_j :=~& \Phi_j^\top Q_j^\top y_j\\
\Sigma_j :=~& \Phi_j^\top Q_j \Phi_j + \mu n I &
Q_j :=~& \Phi_j (\Phi_j^\top \Phi_j + n \lambda I)^{-1} \Phi_j^\top
\end{align*}
\end{lemma}

In the related zero-sum game, the two players take turns playing $w_t \in \Delta(M)$ and $\alpha_t \in \mathbb R^d$ respectively, where $w_t$ is updated according to the Multiplicative Weights Algorithm and $\alpha_t$ is the best response to $w_t$ (see Appendix \ref{app:linear-spaces} for more details). This yields an $O(\sqrt{\log(M)/T} )$-approximate solution to our problem.  
\begin{theorem}[Approximate Solution for Linear Space]\label{thm:approx-sol-linear}
Let $\cH = \{\alpha^\top\phi(\cdot): \alpha \in \R^d, \|\alpha\|_2 \leq A \}$ be a hypothesis space and suppose that $\cF = \cH$.
There is an algorithm that runs for $T$ iterations and returns a $\gamma_T$ additive approximate solution to the regularized empirical min-max adversarial moment problem, where $\gamma_T = O(\sqrt{ \log (M) / T})$ and the runtime for each iteration is $O(d^3 + Md^2)$.
\end{theorem}

\subsection{Reproducing Kernel Hilbert Spaces}
Now, consider the case where the hypothesis class $\cH$ is a Reproducing Kernel Hilbert space (RKHS) and suppose that $\cF = \cH$. Since an RKHS is a closed linear space, the classes $\cH$ and $\cH - \cH$ refer to the same RKHS and thus $\cF\supseteq \shull(\cH-\cH)$.

As with linear spaces, we are able to efficiently obtain an $O(\sqrt{\log(M)/T} )$-approximate solution for the empirical regularized min-max adversarial moment problem.

\begin{theorem}[Approximate Solution for RKHS]
    \label{thm:approx-sol-rkhs}
Let $\cH$ be a compact RKHS and suppose that $\cF = \cH$.
There is an algorithm that runs for $T$ iterations and returns a $\gamma_T$-additive approximate solution to the empirical (regularized) min-max adversarial moment problem %
where  $\gamma_T = O(\sqrt{ \log(M) / T } )$, and the runtime of each iteration is $O(n^3M^4)$.
\end{theorem}

The approximate solution is obtained in a manner similar to the one for linear spaces. We reduce the problem to one that resembles the restated objective in Lemma \ref{lem:rewrite-linear} and solve via no-regret dynamics. In the zero-sum game, the players take turns playing $w_t \in \Delta(M)$ and $\alpha_t \in A \subseteq \mathbb R^{Mn}$, updated in a similar fashion (see Appendix \ref{app:rkhs-spaces} for more details).

\begin{remark}[Finite Dimensional Kernels and Kernel Approximations]\label{rem:poor-scaling-rkhs}
    Kernel methods have the drawback of scaling poorly with sample size due to the large matrix multiplications that are required. However, if the kernel is finite dimensional, or equivalently if one uses a finite dimensional approximation to the kernel matrices (e.g. via the Nystr\"om approximation), then the computation can be considerably alleviated, as suggested in Section \ref{subsec:linear}. 
\end{remark}

\begin{remark}
We separate the operations per iteration into matrix multiplications and inverses. Appealing to the Algorithms \ref{alg:best-response-linear} and \ref{alg:no-regret-rkhs}, we see that a single inversion is computed during the learner's best response and $O(M)$ multiplications are performed during the adversary's play. Note that the matrix operations are no longer for $d  \times d $ matrices, but $(Mn) \times (Mn)$ matrices. However, we can use a Nystr\"om approximation to reduce the dimension of the involved matrices, as noted in Remark \ref{rem:poor-scaling-rkhs}.
\end{remark}

\section{Experiments}
\label{sec:exp}
We compare our adversarial moment approach (Adv-moment) to groupDRO (DRO) and MRO on synthetically generated data where we fit a kernel ridge regression function and a real-world data set (CelebA) where we fit a neural network. We demonstrate the usefulness of our approach by showing that (1) the performance of our approach is better than or comparable to the baselines in terms of average and worst group accuracy (2) the runtime of our approach preserves better runtime as the number of groups increases compared to MRO. 

Code to replicate our experiments can be found at \url{https://github.com/chrisjung/moment_DRO}.

\subsection{Synthetic Data and RKHS}
\label{subsec:exp-rkhs}
To highlight the computational benefits of using our approach, we consider the task of achieving robustness for a synthetic dataset when the hypothesis space is a RKHS with a radial basis function (RBF) kernel: 
\[
\kappa(x, x') = \exp(-\gamma \|x - x\|^2).
\]

As noted in Remark \ref{rem:poor-scaling-rkhs}, kernel methods become costly as the number of samples increases. We circumvent this issue with the Nystr\"om method, approximating the full kernel matrix $K$ with a low rank matrix $\hat{K_r}$ obtained from the following procedure. Let $\hat{x}_1, \dots, \hat{x}_m$ be $m$ samples taken from the entire dataset. Construct the matrix
\[
\hat K_r = K_b \hat K^+ K_b^\top
\]
where $\hat{K} = [ \kappa(\hat{x}_i, \hat{x}_j)]_{m \times m}$ is the kernel matrix obtained from the $m$ samples, $K^+$ is the pseudo-inverse of $\hat{K}$, and $K_b = [ \kappa(x_i, \hat{x}_j)]_{N\times m}$. If we let $\{ (\hat{\lambda}_i, \hat{v}_i), i \in [m] \}$ denote the eigenpairs of $\hat{K}$, $\hat{V_r} = (\hat{v}_1, \dots \hat{v}_r)$ and $\hat{D} = \mathrm{diag}(\hat{\lambda}_1, \dots, \hat{\lambda}_r)$, then we can write the representation of a variable $x$ as $\hat{D}_r^{-1/2} \hat{V}_r^\top (\kappa(x, \hat{x}_1), \dots, \kappa(x, \hat{x}_m))^\top$ (cf.\ \cite{yang2012nystrom}).
This effectively reduces our problem to the linear setting.

\textbf{Data Generation Approach} We choose a setting that showcases the performance gains of MRO and our method in comparison to DRO. We imagine a setting with $2k \geq 2$ groups where half of the groups are defined by having an expected value of $y$ given $x$ as $x^2$, while the other half satisfy $\bE[y|x] = x^2 + 1$. Each half consists of k groups that evenly partition the interval $[-1, 1]$ into $k$ intervals of size $2/k$, with x-values uniformly distributed in each of these intervals. 

We add high levels of gaussian noise with variance drawn uniformly from the interval $[1, 2]$ for the groups with a true function of $x^2$, and add low levels of gaussian noise with variance drawn uniformly from $[0, 0.1]$ (See details of our data generation approach in Appendix~\ref{app:synthetic-details}. Thus, the minmax loss approach of DRO concentrates on high-noise groups and tends toward the lower parabola, while MRO and our moment method choose the parabola situated halfway between both true functions (Figure~\ref{fig:methods-performance}). 

\begin{figure}[ht]
    \centering
    \includegraphics[scale=0.4]{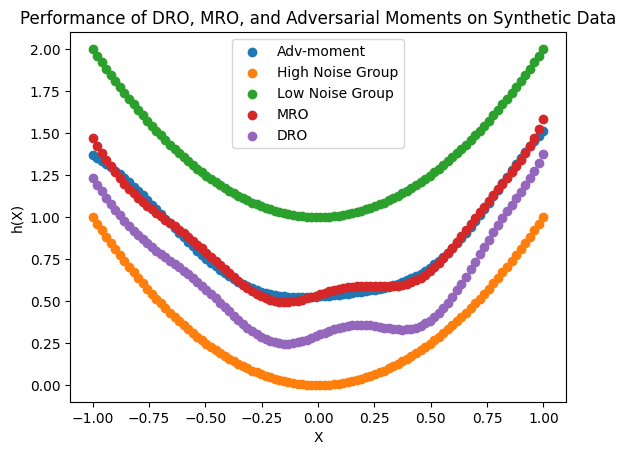}
    \caption{\small Performance of MRO, DRO, and our method (Adv-moment) on synthetic data generated with 50 groups each of size 100. We see that DRO favors the high-noise group while MRO and Adv-moments choose the regret-minimizing parabola halfway between the groups' true functions.}
    \label{fig:methods-performance}
\end{figure}

\textbf{Performance Comparison} In order to showcase the runtime savings of our group, we timed training MRO (square loss), DRO (square loss), and adversarial-moments models while increasing the number of groups (each of size 100) from $2$ (k = 1) to $50$ (k = 25). We use a Nystr\"om approximation of the kernel matrix with 100 components for each method. We see that MRO suffers as the number of groups grows, while our method retains runtime similar to that of DRO (Figure~\ref{fig:methods-runtime}). See Appendix~\ref{app:synthetic-details} for more details. 

\begin{remark}
    We note that runtime comparisons were computed using naive implementations of the MRO and DRO methods as presented in the original papers without taking advantage of any improvements that might be offered by the special choice of RKHS hypothesis spaces. We suspect that the runtime of both methods can be improved upon by leveraging the properties of the matrix multiplications performed at each step, but leave this improvement to future work and present these RKHS experiments as evidence toward the comparable performance of our method.  
\end{remark}

\begin{figure}[ht]
    \centering
    \includegraphics[scale=0.45]{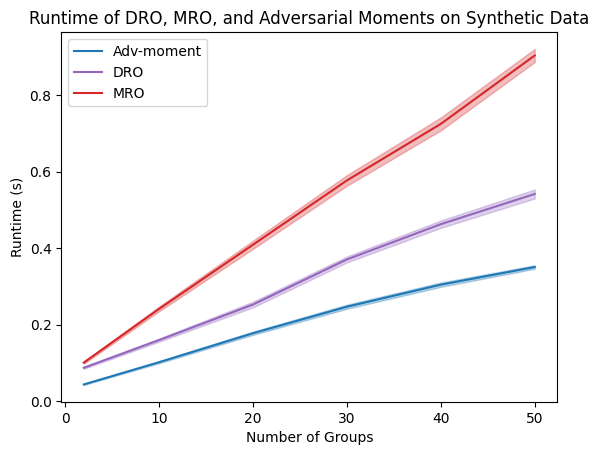}
    \caption{\small Runtime of MRO, DRO, and our method (Adv-moment) on synthetic data generated with increasing numbers of groups. Error bars represent standard error. We see that MRO suffers increasing slowdowns as the number of groups increases.}
    \label{fig:methods-runtime}
\end{figure}

\subsection{CelebA and Neural Network}
\label{subsec:exp-neural}

\begin{table*}
\caption{Performance metrics of DRO, MRO, and our method (Adv-moment) on standard (left) and noisy (right) versions of the CelebA dataset}
\label{tab:celebA}
\vskip 0.15in
\centering
\begin{tabular}{|l|c|c|c|c|c|c|c|c|}
\hline
       & \multicolumn{2}{c|}{Best Epoch} & \multicolumn{2}{c|}{Avg. Acc.} & \multicolumn{2}{c|}{\makecell{Group Acc. Regrets \\(Dark Hair F, Dark Hair M, Blond F, Blond M)}} & \multicolumn{2}{c|}{\makecell{\textbf{Worst Group} \\
       \textbf{Acc. Regret}}}                                            \\ \hline \hline
DRO    & \hphantom{<}5\hphantom{<}  & \hphantom{<}2 & 92.44 & 73.22                 & (8.07, 7.10, 6.77, 10.00) & (3.63, 2.8, 14.56, 28.89)  & 10.00 & 28.89 \\ \hline
MRO    & \hphantom{<}5\hphantom{<} & \hphantom{<}8 & 92.53 & 73.89                 & (8.01, 6.95, 6.77, 10.00) & (3.63, 4.01, 6.77, 10.00) & 10.00 & 10.00 \\ \hline
Moment & \hphantom{<}7\hphantom{<} & \hphantom{<}10 & 89.33 & 71.30                  & (7.42, 13.28, 15.28, 13.33) & (4.66, 8.59, 9.80, 10.00) & 13.33 & 10.00 \\ \hline
\end{tabular}
\vskip -0.1in
\end{table*}

\begin{table*}
\centering
\caption{Runtime of DRO, MRO, and our method (Adv-moment) on the standard CelebA dataset}
\label{tab:celebA-runtime}
\vskip 0.15in%
\begin{tabular}{|l|l|}
\hline
           & Runtime                             \\ \hline
DRO        & 4:49:27                             \\ \hline
MRO        & 5:02:12 (DRO) + 7:27:46 (4 ERM calls) \\ \hline
Adv-Moment & 7:27:43                             \\ \hline
\end{tabular}
\end{table*}

\textbf{Optimization \& Architecture}
The computational approach for the convex hypothesis space implies a heuristic stochastic gradient descent method for neural networks.  

In the convex space case, we maintain an adversary for each distribution $j \in [M]$. In the neural network case, we maintain a single neural network test function $f(X, j; \beta_{t})$ that takes in the features $X$ \emph{and} $j \in [M]$ that denotes which distribution $j \in [M]$ the point is sampled from. Due to the expressivity of neural networks, allowing the model to depend on group features in a non-linear way should result in different outcomes depending on the group feature $j$, roughly translating to training the $M$ separate adversaries we need for each group without actually needing to train $M$ separate networks.  We see empirically that this approach seems to preserve accuracy while reducing runtime compared to MRO.

More specifically, for the adversary, we use resnet50 \citep{he2016deep} to reduce each image to 10 features. These features are then concatenated with the group feature $j \in [M]$. These 11 features go through a two-layer network with a ReLU activation. For the learner, we use a vanilla resnet50 architecture.

We note that due to the adversarial nature of our approach, we must perform both stochastic descent and ascent on each minibatch to train the learner and adversary whereas MRO needs to only perform stochastic descent for each minibatch. Thus, our per-minibatch processing time is roughly doubled compared to MRO, but this trades off with eliminating the need to run ERM on $M$ different distributions and results in overall computational savings that scale better with the number of distributions. 

Now, let $h(\cdot; \alpha_t)$ denote a neural network parameterization of the candidate model. Then we could optimize these networks, as well as the weights $w$ over the populations, using no-regret learning, where we use online gradient ascent for each of the $\beta$ parameters, online gradient descent for the $\alpha$ parameters and exponential weights for $w \in \Delta(M)$. Letting $f^{(\beta)}$ denote the gradient of $f$ with respect to parameters $\beta$ and similarly $h^{(\alpha)}$, the gradient of $h$ with respect to parameters $\alpha$.
\begin{align*}
    U_{t,j,i} :=~& (y_j - h(X_{ij};\alpha_{t}) - f(X_{ij}, j; \beta_{t})) f^{(\beta)}(X_{ij};\beta_{t,j})\\
    V_{t,j,i} :=~& f(X_{ij}, j;\beta_{t})\,  h^{(\alpha)}(X_{ij};\alpha_t)\\
    \beta_{t+1} =~& \beta_{t} + \eta \, \frac{1}{M } \sum_{j \in [M]}\frac{1}{n_j} \sum_{i=1}^{n_j} U_{t,j,i}\\
    w_{t+1,j} \propto~& w_{t,j}\, \exp\left\{\eta_w\, \frac{1}{n_j} \sum_{i=1}^{n_j} U_{t,j,i}\right\}\\
    \alpha_{t+1} =~& \alpha_t - \eta_\alpha\, \sum_{j=1}^M w_{t,j} \frac{1}{n_j} \sum_{i=1}^{n_j} V_{t,j,i}.
\end{align*}
These gradient updates can also be replaced by stochastic versions (See Appendix~\ref{app:exp-neural} for more details). Furthermore, we employ a heuristic for our optimization; if $\cF$ is sufficiently expressive, $f$ that maximizes the adversarial moment violation with respect to a fixed model $h$ would be set to be $f = h - y$. We use this fact as a way to guide our computation of the optimal $f$ at each step by constraining the adversary to output a function of the form $f - h$. The gradient computation is changed as a result of this heuristic, but this is taken care of automatically with the use of automatic gradient computation of PyTorch.

\textbf{Settings: Dataset, Noise-Injection, Architecture }
We focus on the CelebA celebrity dataset \cite{liu2018large}. We follow the same set up as in \citet{sagawa2019distributionally}; the target label is the hair color (blond or dark hair) and the spurious feature is the gender, so there are four groups (blond male, blond female, dark hair male, dark hair female). See \cite{sagawa2019distributionally} for more specific details about the dataset. 

Because the groups of CelebA are defined in terms of their target label, it is not necessarily a good setting to showcase the benefits of minimax regret vs loss because the levels of label noise in each group are identically zero. In order to show the regret savings of our methods in the presence of differential group noise, we perform an experiment on an altered version of the dataset where we inject synthetic noise into a subset of group distributions. For the two majority groups, that is the dark-hair male and female, we change the label of half of their points to be determined by a random coin flip whose probability of a head is 0.5. 

We implemented our approach and MRO by modifying the original groupDRO codebase. We fixed some memory issues and turned off unnecessary gradient computations so that we can measure the runtime of all approaches more faithfully. 

For the learner, we use resnet50 \cite{he2016deep}, and for the adversary, we also use renset50 but modify the last layers. See the appendix for more details. For every approach, we report the test performance metrics on the epoch where the best worst-case group accuracy was achieved on the validation dataset in Table~\ref{tab:celebA}. Also, we include the runtime of each approach. We ran the experiments on an university cluster with access to GPU's.

\subsection{Discussion}
Our RKHS experiment on synthetic data demonstrates that, as expected, our method matches the square-loss regret performance of MRO while improving computation time as the number of groups increases. 
We additionally observe that when used in a deep learning setting to make predictions about the CelebA dataset, our method slightly under-performs in minimizing the group regret compared to the two other methods in the non-noisy version, but improves upon the groupDRO approach and performs comparably to MRO in a setting with differential group noise. Our method also provides a nearly two times faster runtime over the MRO approach due to MRO's need for additional ERM calls to compute the minimum group losses.

\section{Broader Impact}
Distributionally robust learning goals can be seen as a strong guarantee that a model not only does well \emph{overall} on a population, but also does as well as possible on the different subpopulations that make up this larger population. Moreover, our work focuses on minimizing minmax regret rather than minmax loss, which, as noted by \cite{agarwal2022minimax}, ensures that groups are not unduly favored due to high noise when optimizing for a distributionally robust outcome. 

Although our work aims to fix this unfairness issue that may arise as result of heterogeneous levels of noise across different subpopulations, other unfairness issues (e.g. biased dataset) may still need to be addressed.  

\clearpage
\bibliography{ref}
\bibliographystyle{plainnat}

\appendix
\onecolumn

\section{Proofs from Section~\ref{sec:moment-thms}}
\subsection{Theorem~\ref{thm:population-main-mse}}
We restate the theorem for readability:
\begin{theorem}[Population Limit, MSE] Let $h_*$ be a solution to the min-max adversarial moment problem:
\begin{align}
    \min_{h\in H} \max_{D \in \cD}\max_{f \in \cF}\bE_{(X, Y) \sim D}\left[2\,(Y - h(X))\,f(X) - cf(X)^2\right]
\end{align}
with $c > 0$, $\cF \supseteq \shull(H-H) := \{\gamma (h-h'): h,h'\in H, \gamma\in [0,1]\}$. If the pair $H, \cD$ satisfies Assumption~\ref{ass:main}, then for any $c > 0$,
\begin{align*}
    &\max_{D\in \cD} \|h_*-\hz\|_{L^2(D)}^2\\
    &\leq \max \left \{\frac{1}{c}\wedge 1 \right \}\left(\min_{h \in H}\max_{D \in \cD} \|h-\hz\|_{L^2(D)}^2\right) + \epsilon.
\end{align*}
\end{theorem}
\begin{proof}
We begin with the case where $c \geq 1$. Let $h_D$ denote the best approximation to $h_0$ for each distribution $D \in \cD$ as designated by the Assumption~\ref{ass:main}, for which $\|h_0-h_D\|_{L^2(D)}^2\leq \epsilon$. In this case, since, by assumption, $\frac{1}{c}(h_D - h)\in \cF$, we derive, by Corollary~\ref{cor:sandwich-lb}, that for any $h \in H$:
\begin{align}
   &\frac{1}{c}\left(\max_{D\in \cD} \|h_0 - h\|_{L^2(D)}^2 - \epsilon\right)\\
   &\leq \max_{D\in \cD, f \in \cF} \bE_{D}\left[2(Y - h(X))f(X) - cf(X)^2\right].
\end{align}

Thus the minimizer $h_*$ of the minimax adversarial moment satisfies
\begin{align}
   \frac{1}{c}\left(\max_{D\in \cD} \|h_0 - h_*\|_{L^2(D)}^2 - \epsilon\right) &\leq \max_{D\in \cD, f \in \cF} \bE_{D}\left[2(Y - h_*(X))f(X) - cf(X)^2\right] \\
   &\leq \min_{h \in H}\max_{D\in \cD, f \in \cF} \bE_{D}\left[2(Y - h(X))f(X) - cf(X)^2\right] \\
   &\leq \frac{1}{c} \min_{h \in H} \max_{D\in \cD} \|h_0 - h\|_{L^2(D)}^2
\end{align}

Where the last step is an application of Corollary~\ref{cor:sandwich-ub}. Multiplying by $c$ on both sides and adding $\epsilon$ gives the theorem statement (because $c \geq 1$ and thus $\max\{1/c \wedge 1\} = 1$): 

\begin{align}
   \max_{D\in \cD} \|h_0 - h_*\|_{L^2(D)}^2 \leq \min_{h \in H} \max_{D\in \cD} \|h_0 - h\|_{L^2(D)}^2 + \epsilon. 
\end{align}

We now consider the case $0 < c < 1$. By the same upper bound argument, we have that the minimizer $h_*$ of the minimax adversarial moment satisfies
\begin{align}
    &\max_{D\in \cD, f \in \cF} \bE_{D}\left[2(Y - h_*(X))f(X) - cf(X)^2\right] \\
   &\leq \min_{h \in H}\max_{D\in \cD, f \in \cF} \bE_{D}\left[2(Y - h(X))f(X) - cf(X)^2\right] \\
   &\leq \frac{1}{c} \min_{h \in H} \max_{D\in \cD} \|h_0 - h\|_{L^2(D)}^2
\end{align}

However, the lower bound argument does not automatically apply because $1/c > 1$, and thus we are not guaranteed that $\frac{1}{c}(h_D - h) \in \cF$. Instead, we note that because $0 < c < 1$, we have 

\begin{align}
    &\max_{D\in \cD, f \in \cF} \bE_{D}\left[2(Y - h_*(X))f(X) - cf(X)^2\right]\\
    &\geq \max_{D\in \cD, f \in \cF} \bE_{D}\left[2(Y - h_*(X))f(X) - f(X)^2\right]
\end{align}

Here, we can now apply Corollary~\ref{cor:sandwich-lb} to the right-hand-side because $h_D - h \in \cF$ by assumption, giving us 
\begin{align}
   \max_{D\in \cD} \|h_0 - h_*\|_{L^2(D)}^2 - \epsilon
   & \leq \max_{D\in \cD, f \in \cF} \bE_{D}\left[2(Y - h_*(X))f(X) - f(X)^2\right]\\
   &\leq \max_{D\in \cD, f \in \cF} \bE_{D}\left[2(Y - h_*(X))f(X) - cf(X)^2\right] \\
   &\leq \min_{h \in H}\max_{D\in \cD, f \in \cF} \bE_{D}\left[2(Y - h(X))f(X) - cf(X)^2\right] \\
   &\leq \frac{1}{c} \min_{h \in H} \max_{D\in \cD} \|h_0 - h\|_{L^2(D)}^2
\end{align}

Thus adding $\epsilon$ to both sides we recover the theorem statement for $c < 1$ (note that in this case $\max\{1/c \wedge 1\} = 1/c$): 
\begin{align}
   \max_{D\in \cD} \|h_0 - h_*\|_{L^2(D)}^2 
   \leq \frac{1}{c} \min_{h \in H} \max_{D\in \cD} \|h_0 - h\|_{L^2(D)}^2 + \epsilon.
\end{align}
This completes the proof of the theorem.
\end{proof}

\section{Proofs and Additional Statements from Section~\ref{sec:finite-sample}}\label{app:finite-sample-details}
\subsection{Preliminary Lemmas}\label{app:prel}
We will present here a concentration inequality lemma that is used throughout the proofs. See \citep{foster2019orthogonal} for a proof of this theorem:
\begin{lemma}[Localized Concentration, \citep{foster2019orthogonal}]\label{lem:concentration}
For any $h\in \Hcal := \prod_{i=1}^d \Hcal_i$ be a multi-valued outcome function, that is almost surely absolutely bounded by a constant. Let $\ell(Z; h(X))\in \R$ be a loss function that is $L$-Lipschitz in $h(X)$, with respect to the $\ell_2$ norm. Fix any $h_*\in \Hcal$ and let $\delta_n=\Omega\left(\sqrt{\frac{d\,\log\log(n) + \log(1/\zeta)}{n}}\right)$ be an upper bound on the critical radius of $\shull(\Hcal_i-h_{i,*})$ for $i\in [d]$. Then w.p. $1-\zeta$: 
\begin{align}
    \forall h\in \Hcal: \left|(\bE_n - \bE)\left[\ell(Z; h(X)) - \ell(Z; h_*(X))\right]\right| = 18\,L\,\left(d\, \delta_n \sum_{i=1}^d \|h_i - h_{i,*}\|_{L_2} + d\,\delta_n^2\right)
\end{align}
If the loss is linear in $h(X)$, i.e. $\ell(Z; h(X) + h'(X)) = \ell(Z; h(X)) + \ell(Z; h'(X))$ and $\ell(Z;\alpha h(X)) = \alpha \ell(Z;h(X))$ for any scalar $\alpha$, then it suffices that we take $\delta_n=\Omega\left(\sqrt{\frac{\log(1/\zeta)}{n}}\right)$ that upper bounds the critical radius of $\shull(\Hcal_i-h_{i,*})$ for $i\in [d]$.
\end{lemma}

\subsection{Proof of Theorem~\ref{thm:interpolated-reg-mse}}
We restate the theorem for readability:
\begin{theorem} Let $\hat{h}$ be the solution to the empirical regularized min-max adversarial moment problem:
\begin{align}
    \min_{h\in H} \max_{j=1}^M
    L_{j,n}(h) + \mu \|h\|^2
\end{align}
where 
\begin{align}
    L_{j,n}(h) :=~& \max_{f\in F} \bE_{j,n}\left[2(Y - h(X))\, f(X) - a_n f(X)^2\right] - \lambda \|f\|^2
\end{align}
for some $0 < a_n \leq 1/2$ with $\cF=\shull(H-H)$ and $\|\cdot\|$ an arbitrary norm in the space $H-H$. Let $h_j = \argmin_{h\in H} \|h-h_0\|_{L^2(D_j)}^2$ and $F_B=\{f\in F:\|f\|\leq B\}$ for some constant $B$. Let $\delta_n=\Omega\left(\sqrt{\frac{\log\log(n) + \log(M/\delta)}{n}}\right)$ be any solution to the critical radius inequality for $F_B$. Let $h_*=\argmin_{h\in H} \max_{D\in \cD} \|h-h_0\|_{L^2(D)}^2$. Assume that $Y - h_j(X)$ and $Y-h_*(X)$ are a.s. absolutely bounded by $U_Y$ and all $f\in \cF_B$ are a.s. absolutely bounded by $U$. Assume that the pair $H, \cD$ satisfies Assumption~\ref{ass:main}. Then for $\lambda \geq 4 c_0^2 (U_Y \wedge U^2)\delta_n^2/B^2$ and $\mu\geq 4\left(\lambda+c_0^2\delta_n^2 \frac{U^2}{B^2}\right)$, we get that w.p. $1-\delta$:

\begin{align}
    \max_{D\in \cD} \|h_0 - \hat{h}\|_{L^2(D)}^2 
    \leq \frac{8}{a_n}\min_{h\in H} \max_{D\in \cD} \|h_0 - h\|_{L^2(D)}^2 + \frac{10\epsilon}{a_n} + 8\mu \max_j\{\|h_j\|^2 \wedge \|h_*\|^2\wedge 1\}  + O\left(\delta_n^2 (U_Y^2\wedge U^2)/a_n\right)
\end{align}
where $c_0 \leq 18$ is a sufficiently large universal constant.
\end{theorem}

\begin{proof}
Let $h_j=\min_{h\in H} \|h-h_0\|_{L^2(D_j)}^2$. Note that:
\begin{align}
\|h_0 - h\|_{L^2(D_j)}^2 \leq 2 \|h_0 - h_j\|_{L^2(D_j)}^2 + 2 \|h_j-h\|_{L^2(D_j)}^2
\leq 2\epsilon + 2\|h_j-h\|_{L^2(D_j)}^2
\end{align}
By a concentration inequality, for any $h$, such that $\|h_j-h\|\leq B$, we have that if $h_j-h$ is a.s. absolutely bounded by $U$ (wlog $U\geq 1$), when $\|h_j-h\|\leq B$, for a universal constant $1\leq c\leq 18$:
\begin{align}
    \left|(\bE_j-\bE_{j,n})\left[(h_j(X) - h(X))^2\right]\right|\leq c\,U\left(\delta_n \|h_j-h\|_{L^2(D_j)} + \delta_n^2\right)
\end{align}
Applying the latter at any re-scaled $h_D-h$, scaled by $\frac{B}{\|h_D-h\|}$ and scaling back we get that for any $h\in H$:
\begin{align}
    \left|(\bE_j-\bE_{j,n})\left[(h_j(X) - h(X))^2\right]\right|\leq~& c\,U\left(\delta_n \|h_j-h\|_{L^2(D_j)} \frac{\|h-h_j\|}{B} +  \delta_n^2 \frac{\|h-h_j\|_{H}^2}{B^2}\right)\\
    \leq~& \frac{1}{2} \|h_j-h\|_{L^2(D_j)}^2 + 2 c^2 \delta_n^2 \frac{U^2}{B^2}\, \|h-h_j\|^2
\end{align}
This implies that:
\begin{align}
    \|h_j-h\|_{L^2(D_j)}^2\leq~& 2 \bE_{j,n}\left[(h_j(X) - h(X))^2\right] + 4c^2\delta_n^2 \frac{U^2}{B^2}\|h-h_j\|^2\\
    \leq~& 2 \bE_{j,n}\left[2(h_j(X) - h(X))(h_j(X)-h(X)) - (h_j(X)-h(X))^2\right] + 4c^2\delta_n^2 \frac{U^2}{B^2}\|h-h_j\|^2\\
    \leq~& 2 \left(\max_{f\in \cF} \bE_{j,n}\left[2(h_j(X) - h(X))f(X) - f(X)^2\right] - 2\lambda \|f\|^2\right) + 4\left(\lambda + c^2\delta_n^2\frac{U^2}{B^2}\right) \|h-h_j\|^2\\
    \leq~& 2 \left(\max_{f\in \cF} \bE_{j,n}\left[2(h_j(X) - h(X))f(X) - 2a_nf(X)^2\right] - 2\lambda \|f\|^2\right) + 4\left(\lambda + c^2\delta_n^2\frac{U^2}{B^2}\right) \|h-h_j\|^2
\end{align}
(using the assumption that $0 < a_n \leq 1/2$). 
Recall that
\begin{align}
L^{(f)}_{j,n}(h) :=~& \max_{f\in F} \bE_{j,n}\left[2(Y - h(X))\, f(X) - a_n f(X)^2\right] - \lambda \|f\|^2
\end{align}
and note that by the symmetry of the test functions:
\begin{align}
    \sup_{f\in F} \bE_{j,n}\left[2(h_j(X) - h(X))\, f(X) - 2a_nf(X)^2\right] - 2\lambda \|f\|^2 \leq~& \sup_{f\in F} \bE_{j,n}\left[2(h_j(X) - Y)\, f(X) - a_nf(X)^2\right] - \lambda \|f\|^2 \\
    ~&~ + \sup_{f\in F} \bE_{j,n}[2(Y - h(X))\, f(X) - a_n f(X)^2] - \lambda \|f\|^2 \\
    =~& L_{j,n}(h_j) + L_{j,n}(h)
\end{align}
Let $h_*=\argmin_{h\in H}\max_{D\in \cD} \|h-h_0\|_{L^2(D)}^2$. By the empirical optimality of $\hat h$, we are guaranteed that

\[\max_{j} L_{j, n}(\hat h) + \mu\|\hat h \|^2 \leq \max_{j} L_{j, n}(h_*) + \mu\|\hat h \|^2\]
and rearranging terms implies
\[\max_{j} L_{j, n}(\hat h) \leq \max_{j} L_{j, n}(h_*) + \mu(\|\hat h \|^2 - \|\hat h \|^2)\]

Substituting in this upper bound on $L_{j, n}(\hat h)$, we get:
\begin{align}\label{eqn:step2-adv}
    \max_{j}\sup_{f\in F} \bE_{j,n}[2(h_j(X) - \hat{h}(X))\, f(Z) - 2a_nf(Z)^2] - 2\lambda \|f\|^2 \leq \max_{j}L_{j,n}(h_j) + \max_{j} L_{j,n}(h_*)+ \mu (\|h_*\|^2 - \|\hat{h}\|^2)
\end{align}
Thus overall we already have that:

\begin{align}
     \max_j \|h_0 - \hat{h}\|_{L^2(D_j)}^2 \leq 2\epsilon + 4\max_j \left\{L_{j,n}(h_{j}) +  L_{j,n}(h_*)\right\} + 8\left(\lambda+c^2\delta_n^2 \frac{U^2}{B^2}\right)\max_j\{\|\hat{h}-h_j\|^2 \wedge 1\} + 4\mu (\|h_*\|^2 - \|\hat{h}\|^2)
\end{align}

\begin{lemma}\label{lem:unbdd-concentration}
    Let $F_B=\{f\in F: \|f\|\leq B\}$. Let $\delta_n=\Omega\left(\sqrt{\frac{\log\log(n) + \log(M/\delta)}{n}}\right)$ be any solution to the critical radius inequality for $F_B$. Assume $2(Y-h_j(X))$ is a.s. absolutely bounded by $U_Y$ and $f\in F_B$ is a.s. absolutely bounded by $U$. Then for the universal constant $c$ defined in Lemma~\ref{lem:concentration}, w.p. $1-\delta$, for all $f\in F$ and $j\in [M]$:
    \begin{align}
        \left| (\bE_{j,n} - \bE_j)\left[2(Y - h_j(X))\, f(X) - a_nf(X)^2\right]\right| \leq~&  \frac{a_n}{2} \|f\|_{L^2(D_j)}  + c^2 U_Y^2 \delta_n^2/a_n + 3 c^2 (U_Y \wedge U^2) \delta_n^2 \max\{1, \|f\|^2/B^2\}
    \end{align}
\end{lemma}

\begin{proof}
First, note that:
\begin{align}
    \left| (\bE_{j,n} - \bE_j)[2(Y - h_j(X))\, f(X) - a_nf(X)^2]\right| \leq  \left| (\bE_{j,n} - \bE_j)[2(Y - h_j(X))\, f(X)]\right| + \left|(\bE_{j,n} - \bE_j)\left[a_n f(X)^2\right]\right|
\end{align}
By the critical radius inequality, and since $Y-h_j(X)$ is uniformly bounded by $U_Y$, for any $f\in F_B$, we have:
\begin{align}
\left| (\bE_{j,n} - \bE_j)[2(Y - h_j(X))\, f(X)]\right| \leq~& c U_Y (\delta_n \|f\|_{L^2(D_j)} + \delta_n^2)\\
\left|(\bE_{j,n} - \bE_j)\left[a_n f(X)^2\right]\right|\leq~& c a_n U (\delta_n \|f\|_{L^2(D_j)} + \delta_n^2)
\end{align}
For a universal constant $c \geq 1$. 
For any function $f\in \cF$ with $\|f\|\geq B$, applying the latter to $f B / \|f\|$ and scaling back, we get, for any $f\in \cF$:
\begin{align}
\left| (\bE_{j,n} - \bE_j)[2(Y - h_j(X))\, f(X)]\right| \leq~& c U_Y \left(\delta_n \|f\|_{L^2(D_j)} + \delta_n^2 \|f\|/B\right)\leq \frac{a_n}{4} \|f\|_{L^2(D_j)}^2 + c^2 U_Y^2 \delta_n^2/a_n + c U_Y \delta_n^2 \|f\|/B\\
\left|(\bE_{j,n} - \bE_j)\left[a_nf(X)^2\right]\right|\leq~& c a_n U \left(\delta_n \|f\|_{L^2(D_j)} \|f\|/B + \delta_n^2 \|f\|^2/B^2\right)
\leq~ \frac{a_n}{4} \|f\|_{L^2(D_j)}^2 + 2c^2 U^2 \delta_n^2 \|f\|^2/B^2
\end{align}

Thus overall we have that for any $f\in \cF$:
\begin{align}
   \left| (\bE_{j,n} - \bE_j)\left[2(Y - h_j(X))\, f(X) - a_nf(X)^2\right]\right| \leq~&  \frac{a_n}{2} \|f\|_{L^2(D_j)}  + c^2 U_Y^2 \delta_n^2/a_n + 3 c^2 (U_Y \wedge U^2) \delta_n^2 \max\{1, \|f\|^2/B^2\}
\end{align}
\end{proof}

Thus we have that for $c_2=c^2U_Y^2$ and $c_3=3c^2 (U_Y \wedge U^2)$:
\begin{align}
    L_{j,n}(h_j) \leq~& \max_{f} \bE_j[2(Y - h_j(X))\, f(X) - a_n f(X)^2] + \frac{a_n}{2} \|f\|_{L^2(D_j)}^2 + c_2 \delta_n^2/a_n + c_3\, \delta_n^2 \max\{1, \|f\|^2/B^2\} - \lambda \|f\|^2\\
    \leq~& \max_{f} \bE_j[2(h_0(X) - h_j(X))\, f(X) - \frac{a_n}{2}f(X)^2] + c_2 \delta_n^2/a_n + c_3\, \delta_n^2 \max\{1, \|f\|^2/B^2\} - \lambda \|f\|^2\\
    \leq~& \frac{2}{a_n}\bE_j[(h_0(X) - h_j(X))^2] + c_2 \delta_n^2/a_n + c_3\, \delta_n^2 \max\{1, \|f\|^2/B^2\} - \lambda \|f\|^2\\
    \leq~& 2\epsilon/a_n + c_2 \delta_n^2/a_n  + c_3\, \delta_n^2  \max\{1, \|f\|^2/B^2\} - \lambda \|f\|^2\\
    \leq~& 2\epsilon/a_n + c_2 \delta_n^2/a_n + c_3\, \delta_n^2  \tag{$\lambda \geq c_3\, \delta_n^2/B^2$}
\end{align}
Similarly, we have
\begin{align}
    L_{j,n}(h_*) \leq~& \max_{f} \bE_j[2(Y - h_*(X))\, f(X) - a_nf(X)^2] + \frac{a_n}{2} \|f\|_{L^2(D_j)}^2 + c_2 \delta_n^2/a_n + c_3\, \delta_n^2 \max\{1, \|f\|^2/B^2\} - \lambda \|f\|^2\\
    \leq~& \max_{f} \bE_j[2(h_0(X) - h_*(X))\, f(X) - \frac{a_n}{2}f(X)^2] + c_2 \delta_n^2/a_n +c_3\, \delta_n^2 \max\{1, \|f\|^2/B^2\} - \lambda \|f\|^2\\
    \leq~& \frac{2}{a_n}\bE_j[(h_0(X) - h_*(X))^2] + c_2 \delta_n^2/a_n +c_3\, \delta_n^2 \max\{1, \|f\|^2/B^2\} - \lambda \|f\|^2\\
    \leq~& \frac{2}{a_n}\bE_j[(h_0(X) - h_*(X))^2] + c_2 \delta_n^2/a_n +c_3\, \delta_n^2  \tag{$\lambda \geq c_3\, \delta_n^2/B^2$}
\end{align}

Thus we conclude that, with $c_4=c_2+c_3$:
\begin{align}
&\max_j \|h_0 - \hat{h}\|_{L^2(D_j)}^2 \\&\leq \frac{8}{a_n}\max_{j} \|h_0 - h_*\|_{L^2(D_j)}^2 + \frac{10\epsilon}{a_n} + 8\left(\lambda+c^2\delta_n^2 \frac{U^2}{B^2}\right)\max_{j}\{\|\hat{h}-h_j\|^2 \wedge 1\} + 4\mu (\|h_*\|^2 - \|\hat{h}\|^2) + \frac{8 c_4 \delta_n^2}{a_n}
\end{align}
Moreover, we have that:
\begin{align}
    8\left(\lambda+c^2\delta_n^2 \frac{U^2}{B^2}\right)\{\|\hat h-h_j\|^2 \wedge 1\} - 4\mu \|\hat{h}\|^2 \leq~& 8\left(\lambda+c^2\delta_n^2 \frac{U^2}{B^2}\right)\{(2\|\hat{h}\|^2 + 2\|h_j\|^2) \wedge 1\} - 4\mu \|\hat{h}\|^2\\
    \leq~& 16\left(\lambda+c^2\delta_n^2 \frac{U^2}{B^2}\right)\{\|h_j\|^2 \wedge 1\} \tag{$\mu\geq 4\left(\lambda+c^2\delta_n^2 \frac{U^2}{B^2}\right)$}
\end{align}

We can then conclude that:

\begin{align}
    \max_j \|h_0 - \hat{h}\|_{L^2(D_j)}^2 
    \leq~& \frac{8}{a_n}\max_{j} \|h_0 - h_*\|_{L^2(D_j)}^2 + \frac{10\epsilon}{a_n} + 8 \mu \max_j\{\|h_j\|^2 \wedge \|h_*\|^2\wedge 1\} + \frac{8 c_4 \delta_n^2}{a_n} \\
    =~& \frac{8}{a_n}\min_{h\in H} \max_{D\in \cD} \|h_0 - h\|_{L^2(D)}^2 + \frac{10\epsilon}{a_n} + 8\mu \max_j\{\|h_j\|^2 \wedge \|h_*\|^2\wedge 1\}  + \frac{8c_4 \delta_n^2}{a_n}
\end{align}
\end{proof}

\subsection{Alternative Formulation for Bounded Function Spaces}
We present an alternative to Theorem~\ref{thm:interpolated-reg-mse} that assumes $Y$, $\cH$, and $\cF$ are uniformly bounded, but allows for a better multiplicative dependence on the best-in-class error for appropriate choices of parameters. Note that we require an assumption on the critical radius of an additional class $\cG$. We present three versions of the bound, each with varying dependence on $a_n$.

\begin{theorem}\label{thm:interpolate-mse}
Let $\hat{h}$ be the solution to the empirical min-max adversarial moment problem:
\begin{align}
    \min_{h\in H} \max_{j=1}^M \max_{f \in \cF} \frac{1}{n}\sum_{(X_i, Y_i) \in S_j}2(Y_i - h(X_i))f(X_i) - a_nf(X_i)^2
\end{align}
with a function space $\cF$ that satisfies:
    \begin{align}
        \cF \supseteq \shull(H-H) := \{\gamma (h-h'): h,h'\in H, \gamma\in [0,1]\}.
    \end{align}
    Suppose that $Y, h(X), f(X)$ are a.s. absolutely bounded by $U$ (with $U\geq 1$), uniformly over $h\in H, f\in \cF$. Consider the function space:
    \begin{align}
        G = \{x \to \frac{h(x)}{U}\, f(x): h\in H, f\in \cF\}.
    \end{align}
    Let $\delta_n=\Omega\left(\sqrt{\frac{\log\log(n) + \log(M/\zeta)}{n}}\right)$ be an upper bound on the critical radius of $G$ and $F$. Suppose that the pair $H, \cD$ satisfies Assumption~\ref{ass:main}. Then the estimate $\hat{h}$ satisfies w.p. $1-\zeta$ any of the following three inequalities:
    \[\max_{D\in \cD}\|h_0 - \hat{h}\|_{L^2(D)}^2 \leq \frac{2}{a_n}(1 + \eta_2)\min_{h_* \in H}\max_{D\in \cD} \|h_0 - h_*\|_{L^2(D)}^2 + \epsilon + c_5\delta_n^2/\eta_2 \quad\text{(for $a_n \leq 2/3$, $0 < \eta_2 \leq a_n/2$.)},\]
    \[\max_{D\in \cD}\|h_0 - \hat{h}\|_{L^2(D)}^2 \leq 4U\min_{h^* \in H} \max_{D\in \cD} \|h_0 - h_*\|_{L^1(D)} + \epsilon + O(\delta_n^2/a_n),\]
    or
    \[\max_{D\in \cD}\|h_0 - \hat{h}\|_{L^2(D)}^2 \leq (1 + \frac{2(\eta_1 + \eta_2)}{a_n})\min_{h^* \in H} \max_{D\in \cD} \|h_0 - h_*\|_{L^2(D)}^2 + \epsilon + (\frac{1}{\eta_1} + \frac{1}{\eta_2})c_3\delta_n^2\quad \]
for constants $c_4, c_5 = O(U^2).$ The last statement holds for any choice of $0 < \eta_2 \leq a_n/2$, $1 - a_n \leq \eta_1 \leq 1$, $0 < a_n \leq 1$. 
\end{theorem}

\begin{proof}[Proof of Theorem~\ref{thm:interpolate-mse}] 
    Let $\bE_{j}[\cdot] = \bE_{D_j}[\cdot]$ and $\bE_{j,n}[\cdot]=\frac{1}{n}\sum_{(X_i,Y_i)\in S_j}[\cdot]$ denote the empirical average over the samples from $D_j$. 
    
    Since $\delta_n$ upper bounds the critical radius of $H$ and $\cF$ and random variables and functions are bounded, we have
    by a critical radius concentration inequality (see Lemma~\ref{lem:concentration} in Appendix~\ref{app:prel}) and since the loss $2(y-h(X))f(X) - a_nf(X)^2$ is $8U$-Lipschitz in the vector of functions $\frac{h(X)}{U}f(X)$ and $f(X)$, when $Y,h, f$ are uniformly absolutely bounded by $U$, we have that for some universal constant $c_0 \geq 18$ w.p. $1-\delta$ for all $j\in [M]$ and $h\in H, f\in \cF$:
    \begin{align}
        \left| (\bE_{j}-\bE_{j,n})\left[2\,(Y - h(X))\, f(X) - a_nf(X)^2\right] \right| 
        \leq~&  c_0 U \left(\delta_n \left(\sqrt{\bE_{j}\left[\frac{h(X)^2}{U^2} f(X)^2\right]} + \sqrt{\bE_{j}[f(X)^2]}\right) + \delta_n^2\right)\\
        \leq~& c_0 U\left(2 \delta_n \sqrt{\bE_{j}[f(X)^2]} + \delta_n^2\right)\\
        \leq~& \eta_1 \bE_{j}[f(X)^2] + \frac{1}{\eta_1} c_1 \delta_n^2 \tag{AM-GM inequality}
    \end{align}
    for any choice of $\eta_1 \in (0,1]$, with $c_1 = O(U^2)$.
    Thus, for any $0 < \eta_1 \leq 1$, we get, w.p. $1-\delta$:
    \begin{align}
        \max_{j\in [M]} \bE_j\left[2\,(Y - h(X))\, f(X) - (a_n+\eta_1)\,f(X)^2\right]
        \leq \max_{j\in [M]} \bE_{j,n}\left[2\,(Y - h(X))\, f(X) - a_nf(X)^2\right] + \frac{1}{\eta_1}c_1 \delta_n^2
    \end{align}
    Applying the latter at $h\to \hat{h}$ and invoking the optimality of $\hat{h}$ and the fact that $\frac{1}{a_n +\eta_1}(h_D-h)\in \cF$ for any $1 - a_n \leq \eta_1 \leq 1$, we have that for any $h_*\in H$:
    \begin{align}
        \frac{1}{a_n+\eta_1} \left(\max_{j \in [M]}\|h_0 - \hat{h}\|_{L^2(D_j)}^2 - \epsilon\right)
        \leq~& \max_{j\in [M]}\max_{f \in \cF}\bE_{j,n}\left[2\,(Y - \hat{h}(X))\, f(X) - a_nf(X)^2\right] + \frac{1}{\eta_1} c_1 \delta_n^2
        \\
        \leq~& \max_{j\in [M]}\max_{f \in \cF}\bE_{j,n}\left[2\,(Y - h_*(X))\, f(X) - a_nf(X)^2\right] + \frac{1}{\eta_1} c_1 \delta_n^2
    \end{align}
    where the first inequality is an application of Corollary~\ref{cor:sandwich-lb}.
    By the same critical radius concentration inequality (Lemma~\ref{lem:concentration}, Appendix~\ref{app:prel}), for any $\eta_2 \in (0, 1]$, $c_2 = O(U^2)$, w.p. $1-\delta$, for all $j\in [M]$ and $f \in \cF$:
    \begin{align}
        \bE_{j,n}\left[2\,(Y - h_*(X))\, f(X) - a_nf(X)^2\right]
        \le~& \bE_j\left[2\,(Y-h_*(X))f(X) - (a_n -\eta_2) f(X)^2\right] + \frac{1}{\eta_2} c_2 \delta_n^2 \label{eqn:moment_to_l1_ub}\\
        \leq~& \frac{1}{a_n -\eta_2} \|h_0 - h_*\|_{L^2(D_j)}^2 + \frac{1}{\eta_2} c_2\delta_n^2
    \end{align}
    where the last step is an application of Corollary~\ref{cor:sandwich-ub} and holds for any $0 < \eta_2 < a_n$. 
    
    Thus we conclude that for any $h_*\in H$, and taking $\eta_2 \leq a_n/2$:
    \begin{align}
    \max_{D\in \cD}\|h_0 - \hat{h}\|_{L^2(D)}^2
    \leq~& \frac{a_n +\eta_1}{a_n -\eta_2} \max_{D\in \cD} \|h_0 - h_*\|_{L^2(D)}^2 + \epsilon + (a_n + \eta_1)(\frac{1}{\eta_1} c_1\delta_n^2 + \frac{1}{\eta_2} c_2\delta_n^2) \\
    \leq~& \left(1 + \frac{2(\eta_1 + \eta_2)}{a_n}\right) \max_{D\in \cD} \|h_0 - h_*\|_{L^2(D)}^2 + \epsilon + (\frac{1}{\eta_1}  + \frac{1}{\eta_2})c_4\delta_n^2 \\
    \end{align}
    for $c_4 = O(U^2)$, and any choice of $1 - a_n \leq \eta_1 \leq 1$, $0 < \eta_2 \leq a_n/2$.
    
    Since $h_*$ was arbitrarily chosen, we get:
    \begin{align}
    \max_{D\in \cD}\|h_0 - \hat{h}\|_{L^2(D)}^2 
    \leq \left(1 + \frac{2(\eta_1 + \eta_2)}{a_n}\right) \min_{h_* \in H}\max_{D\in \cD} \|h_0 - h_*\|_{L^2(D)}^2 + \epsilon + (\frac{1}{\eta_1}  + \frac{1}{\eta_2})c_4\delta_n^2.
    \end{align}
    
    We note that for $a \leq 2/3$ and taking $\eta_1 \leq 1 - a_n/2$, we can get a simplified bound of 
    \begin{align}
    \max_{D\in \cD}\|h_0 - \hat{h}\|_{L^2(D)}^2 
    \leq \frac{2}{a_n}(1 + \eta_2)\min_{h_* \in H}\max_{D\in \cD} \|h_0 - h_*\|_{L^2(D)}^2 + \epsilon + c_5\delta_n^2/\eta_2.
    \end{align}
    
    with $c_5 = O(U^2)$, for any choice of $a_n \leq 2/3$ and $0 < \eta_2 \leq a_n/2$.  
    
    To recover a bound in terms of the best-in-class $\ell_1$ error, we return to equation~\ref{eqn:moment_to_l1_ub} and note that rather than bounding the mean square error, we can give an alternative upper bound of 
    
    \begin{align}
        \bE_j\left[2\,(h_0(X)-h_*(X))f(X) - (a_n -\eta_2) f(X)^2\right] + \frac{1}{\eta_2} c_2 \delta_n^2 \leq~& \bE_j\left[2\,(h_0(X)-h_*(X))f(X)\right] + \frac{1}{\eta_2} c_2 \delta_n^2 \\
        \leq~& 2U||h_0 - h_*||_{L^1(D_j)} + \frac{1}{\eta_2} c_2 \delta_n^2\\
        \leq~& 2U||h_0 - h_*||_{L^1(D_j)} + \frac{1}{a_n} c_6 \delta_n^2
    \end{align}
    
    For $c_6 = O(U^2)$. This results in a bound in terms of the best-in-class hypothesis' $\ell_1$ error taking $\eta_1 = 1$:
    \begin{align}
    \max_{D\in \cD}\|h_0 - \hat{h}\|_{L^2(D)}^2 
    \leq~& (a_n + \eta_1)2U\min_{h^* \in H} \max_{D\in \cD} \|h_0 - h_*\|_{L^1(D)} + \epsilon + (a_n + \eta_1)(\frac{1}{\eta_1}c_1\delta_n^2 + \frac{1}{a_n}c_6\delta_n^2)\\
    \leq~& 4U\min_{h^* \in H} \max_{D\in \cD} \|h_0 - h_*\|_{L^1(D)} + \epsilon + O(\delta_n^2/a_n)\\
    \end{align}
    for $a_n \leq 1$.
    \end{proof}

\subsection{\texorpdfstring{$\ell_2$}{TEXT}-penalized Multiaccuracy Guarantees}\label{app:ma-guarantees}
We can also derive finite sample guarantees for the $\ell_2$-penalized multiaccuracy objective. Note that due to our population-limit upper bound on the moment objective (Corollary~\ref{cor:sandwich-ub}), we can use either Theorem~\ref{thm:interpolated-reg-mse} or Theorem~\ref{thm:interpolate-mse} to derive an upper bound to the $\ell_2$-penalized accuracy objective in terms of the best-possible square error. However, this bound is somewhat loose and incurs a multiplicative factor of $\frac{1}{a_n^2}$. 

We additionally present a tighter bound that directly bounds the multiaccuracy error when $\cY, \cH, $ and $\cF$ are uniformly bounded:
\begin{theorem}[Finite Sample Bounds:  multiaccuracy]\label{thm:interpolate-ma}
Let $\hat{h}$ be the solution to the empirical min-max adversarial moment problem:
\begin{align}
    \min_{h\in H} \max_{j=1}^M \max_{f \in \cF} \frac{1}{n}\sum_{(X_i, Y_i) \in S_j}2(Y_i - h(X_i))f(X_i) - a_nf(X_i)^2
\end{align}
with a function space $\cF$ that satisfies:
    \begin{align}
        \cF \supseteq \shull(H-H) := \{\gamma (h-h'): h,h'\in H, \gamma\in [0,1]\}.
    \end{align}
    Suppose that $Y, h(X), f(X)$ are a.s. absolutely bounded by $U$ (with $U\geq 1$), uniformly over $h\in H, f\in \cF$. Consider the function space:
    \begin{align}
        G = \{x \to \frac{h(x)}{U}\, f(x): h\in H, f\in \cF\}.
    \end{align}
    Let $\delta_n=\Omega\left(\sqrt{\frac{\log\log(n) + \log(M/\zeta)}{n}}\right)$ be an upper bound on the critical radius of $G$ and $F$. Suppose that the pair $H, \cD$ satisfies Assumption~\ref{ass:main}. Then the estimate $\hat{h}$ satisfies w.p. $1-\zeta$:

    \[\max_{f \in \cF}\max_{j\in [M]} \bE_j\left[(h_0(X) - \hat h(X))\, f(X)\right] \leq \min_{h_* \in \cH}\max_{f \in \cF}\max_{j\in [M]} \bE_j\left[(h_0(X)-h_*(X))f(X)\right] + O(a_n + \delta_n).\]
\end{theorem}

\begin{proof}[Proof of Theorem~\ref{thm:interpolate-ma}]
Let $\bE_{j}[\cdot] = \bE_{D_j}[\cdot]$ and $\bE_{j,n}[\cdot]=\frac{1}{n}\sum_{(X_i,Y_i)\in S_j}[\cdot]$ denote the empirical average over the samples from $D_j$. Since $\delta_n$ upper bounds the critical radius of $H$ and $\cF$ and random variables and functions are bounded, we have
by a critical radius concentration inequality (see Lemma~\ref{lem:concentration} in Appendix~\ref{app:prel}) and since the loss $2(y-h(X))f(X) - a_nf(X)^2$ is $8U$-Lipschitz in the vector of functions $\frac{h(X)}{U}f(X)$ and $f(X)$, when $Y,h, f$ are uniformly absolutely bounded by $U$, we have that for some universal constant $c_0$ w.p. $1-\delta$ for all $j\in [M]$ and $h\in H, f\in \cF$:
\begin{align}
    \left| (\bE_{j}-\bE_{j,n})\left[2\,(Y - h(X))\, f(X) - a_nf(X)^2\right] \right| 
    \leq~&  c_0 U \left(\delta_n \left(\sqrt{\bE_{j}\left[\frac{h(X)^2}{U^2} f(X)^2\right]} + \sqrt{\bE_{j}[f(X)^2]}\right) + \delta_n^2\right)\\
    \leq~& c_0 U\left(2 \delta_n \sqrt{\bE_{j}[f(X)^2]} + \delta_n^2\right)\\
    \leq~& c_1\delta_n||f(X)||_{L_2(D_j)} + c_2\delta_n^2
\end{align}
with $c_1, c_2 = O(U)$. Thus, we get 
\begin{align}
    &\bE_j\left[2\,(Y - h(X))\, f(X) - a_n\,f(X)^2\right]\\
    &\leq \bE_{j,n}\left[2\,(Y - h(X))\, f(X) - a_nf(X)^2\right] + c_1\delta_n||f(X)||_{L_2(D_j)} + c_2\delta_n^2
\end{align}
For any $j \in [M]$, $h \in \cH$, $f \in \cF$, which allows us to rewrite as 
\begin{align}
    &\max_{j\in [M]} \bE_j\left[(Y - h(X))\, f(X)\right]\\
    &\leq \frac{1}{2}\max_{j\in [M]} \left(\bE_{j,n}\left[2\,(Y - h(X))\, f(X) - a_nf(X)^2\right] + a_n\mathbb{E}_j[f(X)^2] + c_1\delta_n||f(X)||_{L_2(D_j)} \right)+ c_2\frac{1}{2}\delta_n^2
\end{align}
Applying the latter at $h\to \hat{h}$ and invoking the optimality of $\hat{h}$, we have that for any $h_*\in H$:
\begin{align}
    \max_{j\in [M]}\max_{f \in \cF}\bE_{j,n}\left[2\,(Y - \hat{h}(X))\, f(X) - a_nf(X)^2\right] 
    \leq \max_{j\in [M]}\max_{f \in \cF}\bE_{j,n}\left[2\,(Y - h_*(X))\, f(X) - a_nf(X)^2\right]
\end{align}
By the same critical radius concentration inequality, w.p. $1-\delta$, for all $j\in [M]$ and $f \in \cF$:
\begin{align}
    \bE_{j,n}\left[2\,(Y - h_*(X))\, f(X) - a_nf(X)^2\right]
    \le~& \bE_j\left[2\,(Y-h_*(X))f(X) - a_nf(X)^2\right] + c_1\delta_n||f(X)||_{L_2(D)} + c_2\delta_n^2\\
    =~& \bE_j\left[2\,(h_0(X)-h_*(X))f(X) - a_n f(X)^2\right] + c_1\delta_n||f(X)||_{L_2(D)} + c_2\delta_n^2
\end{align}

Combining these inequalities gives us
\begin{align}
    &\max_{f \in \cF}\max_{j\in [M]} \bE_j\left[(h_0(X) - \hat h(X))\, f(X)\right]\\
    &\leq \min_{h_* \in \cH}\max_{f \in \cF}\max_{j\in [M]} \bE_j\left[(h_0(X)-h_*(X))f(X)\right] + \max_{j \in [M]}\max_{f \in \cF}\left(\frac{a_n}{2}\mathbb{E}_j[f(X)^2] + c_1\delta_n||f(X)||_{L_2(D_j)}\right) + c_2\delta_n^2\\
    &\leq \min_{h_* \in \cH}\max_{f \in \cF}\max_{j\in [M]} \bE_j\left[(h_0(X)-h_*(X))f(X)\right] + c_4a_n + c_3\delta_n + c_2\delta_n^2
\end{align}

with $c_3 = O(U), c_4 = O(U^2)$. Thus, we conclude that 

\[\max_{f \in \cF}\max_{j\in [M]} \bE_j\left[(h_0(X) - \hat h(X))\, f(X)\right] \leq \min_{h_* \in \cH}\max_{f \in \cF}\max_{j\in [M]} \bE_j\left[(h_0(X)-h_*(X))f(X)\right] + O(a_n + \delta_n).\]
\end{proof}

\subsection{Linear Function Spaces}\label{app:lin-samp-cxty}
Similar results can be obtained for the norm-contrained un-regularized variant of the finite sample bound for linear function spaces given in Section~\ref{sec:finite-sample}, i.e. where $\cH_B=\{\alpha^\top\phi(\cdot): \alpha \in \R^d, \|\alpha\|_2\leq B\}$, for some constant $B$. 

In this case, note that we can take $\cF=\cH_{2B}$. Moreover, by a volumetric argument (cf.\ Example 13.8 in \citet{wainwright2019high}), the critical radius of the function space $\cF$ is bounded as $O\left(\frac{d}{n} B V\right)$ and by a similar volumetric argument (see e.g.\ Appendix D in \cite{dikkala2020minimax}) the critical radius of the function space $G_B:=\{x\mapsto \frac{\theta^\top\phi(x)}{B\, V} \cdot \beta^\top\phi(x): \|\theta\|_2\leq B, \|\beta\|_2\leq 2B\}$ is bounded as $O\left(\frac{d^2}{n} B V\right)$. If we further assume that $|Y|\leq O(V B)$ and invoking Theorem~\ref{thm:interpolate-mse} we get a rate of:
\begin{align*}
    \max_{D\in \cD} \|h_0 - \hat{h}\|_{L^2(D)}^2 
    \leq &(1+4\eta)\min_{h\in \cH_B} \max_{D\in \cD} \|h_0 - h\|_{L^2(D)}^2 \\
    &+ \epsilon + O\left(\frac{1}{\eta}\frac{d^2}{n} V^3  B^3\right)
\end{align*}
Comparing the two bounds, we see that the first bound adapts to the $\ell_2$ norm of the minimizers, i.e. $\|\alpha_j\|_2^2$ and $\|\alpha_*\|_2^2$, instead of being dependent on a worst-case bound of $B$ on the norm of a good approximate solution in $\cH$. Moreover, the approximation error in the second theorem is with respect to the norm-constrained function space $\cH_B$, while in the first statement the approximation error $\epsilon$ and $\min_{h\in \cH}\max_{D\in \cD} \|h_0-h\|$ is computed over the un-constrained class $\cH$, and can be substantially smaller. Moreover, the second theorem statement has a slightly larger critical radius due to the dependence on the radius of the function space $G_B$. However, the second statement has a much smaller constant multiplying the min-max solution, which can be taken as close to $1$ as one would want.

\section{Missing Details from Section \ref{sec:comp}}
\subsection{Preliminary Lemma}
We make use of the following lemma in our proofs.
\begin{theorem}[Minimax Theorem, \cite{sion1958general}]\label{thm:sion-minmax}
Let $X$ and $Y$ be nonvoid convex and compact subsets of two linear topological spaces, and let $f: X \times Y \to \mathbb R$ be a function that is upper semicontinuous and quasi-concave in the first variable and lower semicontinuous and quasi-convex in the second variable. Then 
\begin{align*}
    \min_{y \in Y} \max_{x \in X} f(x, y) = \max_{x \in X} \min_{y \in Y} f (x, y)
\end{align*}
\end{theorem}

\subsection{No-Regret Dynamics}
\label{app:learning-dynamics}

We can frame the empirical (regularized) min-max adversarial moment problem as the problem of finding an equilibrium of a zero-sum game between two players: the Learner and the Adversary. 
The Learner's action space consists of hypotheses  $h \in \cH$ and the Adversary's action space consists of  a tuple of weight and test function $(w, f) \in \Delta(M) \times \cF$. 
The players are respectively trying to minimize and maximize a 
(regularized) weighted moment violation criterion:
\begin{align*}
\mathcal L (h, (w, f)) = \sum_{j=1}^M w_j \bE_{j, n} [2 (Y - h(X))f(X) - f(X)^2 ]  - \lambda \|f \|^2 + \mu \|h \|^2.
\end{align*}
By convexity of $\Delta(M)$, finding a min-max equilibrium of $\mathcal L$ is equivalent to our initial optimization problem:
\begin{align*}
\arg \min_{h \in \cH} \max_{w \in \Delta(M)} \max_{f \in \cF} \mathcal L (h, (w, f)) 
=
 \arg \min_{h \in \cH} \max_{j \in [M]} \max_{f \in \cF} \bE_{j, n} [2 (Y - h(X))f(X) - f(X)^2 ] - \lambda \|f\|^2 + \mu \|h\|^2 
\end{align*}

A classic result of \cite{freund1999adaptive} tells us that an approximate min-max equilibrium of $\mathcal L$ can be obtained by simulating \textit{no-regret} dynamics between the Learner and Adversary. 
In each round $t$ of our game, the Adversary plays a tuple $(w_t, f_t)$ before the Learner plays a hypothesis $h_t$. By virtue of playing second, the Learner is able to best respond to the Adversary. The Adversary on the other hand has to anticipate the Learner's play. In order to best respond, the Adversary seemingly would have to enumerate over all the hypotheses and compute the response for each hypothesis--a computationally intractable approach. Instead, the Adversary updates its parameters $(w_t, f_t)$ according to no-regret learning algorithms (e.g. Multiplicative Weights, Follow the Leader). 

With the regret guarantees of the Learner and Adversary, we bound the error associated with the approximate solution.

\begin{theorem}[\cite{freund1999adaptive}, No-regret Dynamics, Restated]\label{thm:approx-equilibrium}
Let $\{h_t\}_{t=1}^T \in \cH^T$ be a sequence of actions played by the Learner and $\{(w_t, f_t)\}_{t=1}^T \in (\Delta(M) \times \cF)^T$ be a sequence played by the Adversary. Let $\bar{h} = \frac{1}{T} \sum_{t=1}^T h_t$ and $(\bar{w}, \bar{f}) = (\frac{1}{T}\sum_{t=1}^T w_t, \frac{1}{T}\sum_{t=1}^T f_t)$ be the empirical distribution over their sequence of actions. 
Suppose 
\begin{align*}
    \sum_{t=1}^T \mathcal L(h_t, (w_t, f_t)) - \min_{h^* \in \cH} \sum_{t=1}^T \mathcal L (h^*, (w_t, f_t)) &\le \gamma_L T \\
    \max_{(w^*, f^*) \in \Delta(M) \times \cF} \sum_{t=1}^T \mathcal L (h_t, (w^*, f^*)) - \sum_{t=1}^T \mathcal L (h_t, (w_t, f_t)) &\le \gamma_A T
\end{align*}
Then, we have $(\bar{h}, (\bar{w}, \bar{f}))$ is an $(\gamma_L + \gamma_A)$-approximate equilibrium:
\begin{align*}
    \max_{(w^*, f^*) \in \Delta(M) \times \cF}  \mathcal L(\bar{h}, (f^*, w^*)) - (\gamma_L + \gamma_A) &\le 
    \min_{h^* \in \cH} \max_{(w^*,f^*) \in \Delta(M) \times \cF} \mathcal L (h^*, (w^*, f^*))  \\
\min_{h^* \in \cH} \mathcal L(h^*, (\bar{w}, \bar{f})) 
+ (\gamma_A + \gamma_L) 
&\geq 
\min_{h^* \in \cH} \max_{(w^*,f^*) \in \Delta(M) \times \cF} \mathcal L (h^*, (w^*, f^*)) 
\end{align*}
\end{theorem}
Implicit in our application of the above theorem is the use of the minimax theorem of \cite{sion1958general}, which allows us to rewrite a min-max problem as a max-min problem. Note that in the min-max adversarial moment problem, we have a min-max-max objective, where the loss function is not concave in $f$. To circumvent this, we first solve for $f$ in closed form in the settings we consider (either directly, or with an oracle). With $f$ fixed, we can then apply the min-max theorem as the loss function becomes convex in $h$ and concave in $w$.

In the rest of this section, we fill in the details for various instantiations of $\cH$. In each subsection, we do the following. First, we use our assumptions on $\cH$ to restate the optimization problem in a more convenient form. Second, we show how the Learner and Adversary update their parameters. Finally, we show that the updates are no-regret and give an approximate solution of the empirical  (regularized) min-max adversarial moment problem. 

\subsection{Linear Hypothesis Spaces}\label{app:linear-spaces}
Let $\cH = \{\alpha^\top\phi(\cdot): \alpha \in \R^d, \|\alpha\|_2 \leq A \}$ be a class containing linear functions for a known feature map $\phi(\cdot)$ and suppose $\cF = \cH$. Note that $\cH-\cH$ is the same space of linear functions, so we also have $\cF\supseteq \shull(\cH-\cH)$. For a hypothesis $h \in \cH$ and index $j \in [M]$, define the moment violation
\[ 
L_{j,n}(h, f):=  \bE_{j, n} [2 (Y - h(X))f(X) - f(X)^2 ] .
\]

To solve the empirical (regularized) min-max adversarial moment problem, we first rewrite the objective. 
\begin{lemma}[Adversarial Moment Problem, Restated for Linear Spaces]\label{lem:linear-restated-appendix}
Let $X_{ij}, Y_{ij}$ denote the $i$-th sample from the $j$-th dataset. Furthermore, let  $y_j = [Y_{1j}; \dots; Y_{nj}] \in \mathbb R^{n}$ and $\Phi_j = [\phi(X_{1j})^\top; \dots; \phi(X_{nj})^\top] \in \mathbb R^{n \times d}$. Then
\begin{align*}
\nonumber
 &\min_{h \in \cH} \max_{j \in [M]} \max_{f \in \cF} L_{j,n}(h, f) - \lambda \|f\|^2 + \mu \|h\|^2  \\
=~&\min_{\substack{\alpha \in \mathbb R^d \\ \|\alpha\|_2 \leq A}} \max_{w \in \Delta(M)}  \frac{1}{n} \sum_{j=1}^M w_j \left (\kappa_j - 2 \nu_j^\top \alpha + \alpha^\top \Sigma_j \alpha \right ) 
\end{align*}
where
\begin{align*}
\kappa_j :=~& y_j^\top Q_j y_j  &
\nu_j :=~& \Phi_j^\top Q_j^\top y_j \\
\Sigma_j :=~& \Phi_j^\top Q_j \Phi_j + \mu n I &
Q_j :=~& \Phi_j (\Phi_j^\top \Phi_j + n \lambda I)^{-1} \Phi_j^\top
\end{align*}
\end{lemma}

\begin{proof}[Proof of Lemma \ref{lem:linear-restated-appendix}]
Fix a hypothesis $h(x) = \alpha^\top \phi(x)$ and an index $j \in [M]$.

Observe that 
\begin{align*}
\max_{f \in \cF}
L_{j,n}(h, f) 
&=
\max_{f \in \cF} \bE_{j, n} [2 (Y - h(X))f(X) - f(X)^2 ] - \lambda \|f \|^2 + \mu \|h \|^2 \\
&=   
\max_{f \in \cF}
 \bE_{j,n}[2Yf(X)] - \bE_{j,n}[2h(X)f(X)] - \bE_{j,n}[f(X)^2] - \lambda \|f\|^2 + \mu \|h\|^2 \\
&=  
\max_{\beta \in \mathbb R^d}
 \bE_{j, n}[2Y \phi(X)^\top] \beta -  \alpha^\top \bE_{j, n}[2\phi(X) \phi(X)^\top]\beta - \beta^\top \bE_{j, n}[\phi(X) \phi(X)^\top] \beta - \lambda \beta^\top \beta + \mu \alpha^\top \alpha \\
&= 
\max_{\beta \in \mathbb R^d}
\frac{2}{n} y_j^\top \Phi_j \beta 
- \frac{2}{n} \alpha^\top \Phi_j^\top \Phi_j \beta
- \frac{1}{n} \beta^\top \Phi_j^\top \Phi_j \beta
- \lambda \beta^\top \beta 
+ \mu \alpha^\top \alpha \\
&=
\max_{\beta \in \mathbb R^d}
\frac{2}{n} \left ( y_j^\top \Phi_j - \alpha^\top \Phi_j^\top \Phi_j \right ) \beta 
- \frac{1}{n} \beta^\top \left (\Phi_j^\top \Phi_j + n \lambda I \right ) \beta
+ \mu \alpha^\top \alpha
\end{align*}
The first-order condition for $\beta$ tells us that
\begin{align*}
0 &= 
\frac{2}{n} \left ( \Phi_j^\top y_j -\Phi_j^\top \Phi_j \alpha \right) 
- \frac{2}{n} \left (\Phi_j^\top \Phi_j + n \lambda I \right ) \beta
\end{align*}
Thus, we conclude that the optimal value is attained when $\beta = (\Phi_j^\top \Phi_j + n \lambda I)^{-1}(\Phi_j^\top y_j - \Phi_j^\top \Phi_j \alpha)$.
Substituting this value in our objective, we derive that 
\begin{align*}
\max_{f \in \cF}
L_{j,n}(h, f) 
&=
\frac{1}{n}(\Phi_j^\top y_j - \Phi_j^\top \Phi_j \alpha)^\top (\Phi_j^\top \Phi_j + n \lambda I)^{-1}(\Phi_j^\top y_j - \Phi_j^\top \Phi_j \alpha)
+ \mu \alpha^\top \alpha \\
&= 
\frac{1}{n} 
\left(\
y_j^\top Q_j y_j
- 2 (\Phi_j^\top Q_j^\top y_j)^\top \alpha 
+ \alpha^\top  ( \Phi_j^\top Q_j \Phi_j + \mu n I ) \alpha
\right )
\end{align*}
where $Q_j = \Phi_j (\Phi_j^\top \Phi_j + n \lambda I)^{-1} \Phi_j^\top$.
Recalling the definition of $\kappa_j, \Sigma_j$ and $\nu_j$, we also see that 
\begin{align*}
\max_{f \in \cF}
L_{j,n}(h, f) =
\frac{1}{n}\left (\kappa_j - 2 \nu_j^\top \alpha + \alpha^\top \Sigma_j \alpha \right ) 
\end{align*}

Our desired result immediately follows from the defintion of $\cH$ and the convexity of $\Delta(M)$.
\end{proof}

We therefore end up with a convex-concave min-max optimization problem that can be solved via no-regret dynamics. We find an approximate solution to our desired problem using no-regret dynamics, where the Learner and Adversary are respectively trying to minimize and maximize the objective
\[
\mathcal L(\alpha, w) := \sum_{j = 1}^M w_j ( \kappa_j - 2 \nu_j^\top \alpha + \alpha^\top \Sigma_j \alpha).
\]
The Learner will play a vector $\alpha \in \mathbb R^{d}$ and the Adversary will play a weight $w \in \Delta(M)$.

\subsubsection*{The Learner's Best Response}

In each round $t$, the Learner best responds to the Adversary's action $w_t$ by finding   $\alpha_t \in \mathbb R^{d}$ 
\begin{align*}
\alpha_t = \argmin_{\alpha \in A} \sum_{j = 1}^M w_{t,j} ( \kappa_j - 2 \nu_j^\top \alpha + \alpha^\top \Sigma_j \alpha)
\end{align*}
Note that the first order condition tells us that 
\begin{align*}
    \left ( \sum_{j = 1}^M w_{t,j} \Sigma_j \right ) \alpha = \sum_{j = 1}^M w_{t,j}\nu_j.
\end{align*}
This suggests a natural approach to get the best response, which we describe in Algorithm \ref{alg:best-response-linear}.
\begin{algorithm}[ht]
   \caption{Best Response, $\mathrm{BEST}(w)$, for Learner}
   \label{alg:best-response-linear}
\begin{algorithmic}
   \STATE \textbf{Input:} weight vector $w$, and  matrices $\nu_j, \Sigma_j$
   \STATE{Set $\overline{\nu} = \sum_{j = 1}^{M} w_{j} \nu_j$ }
   \STATE{Set $\overline{\Sigma} = \sum_{j = 1}^M w_j \Sigma_j$ }
   \STATE{Set $\alpha = \overline{\Sigma}^{+} \overline{\nu}$, where $\overline{\Sigma}^+$ is the pseudo-inverse of  the matrix $\overline{\Sigma}$}
   \STATE{ \bfseries Output:} $\alpha$
\end{algorithmic}
\end{algorithm}

\begin{lemma}[Learner's Regret]\label{lem:best-response-linear}
Let $w_1, \dots, w_T \in \Delta(M)$ be a sequence of weights and let $\alpha_1, \dots, \alpha_T \in \mathbb R^{Mn}$ be a sequence such that 
$\alpha_t = \mathrm{BEST}(w_t)$ for each $t \in [T]$.
Suppose that for each weight $w_t$, the matrix $\sum_{j = 1}^{M} w_{t,j} \Sigma_j$ is invertible. Then 
\[
\sum_{t = 1}^T \mathcal L(\alpha_t, w_t)
- \min_{\alpha \in \mathbb R^{Mn} }\sum_{t = 1}^T \mathcal L (\alpha, w_t) \leq 0
\] 
\end{lemma}

\subsubsection*{The Adversary's No Regret Updates}

In each round $t$, the Adversary plays the weight $w_t$ according to the Multiplicative Weights algorithm. Refer to Algorithm \ref{alg:no-regret-rkhs} for more details.

\begin{algorithm}[hb]
\caption{No-Regret Dynamics}
\label{alg:no-regret-rkhs}
\begin{algorithmic}
\STATE {\bfseries{Input: }} Time horizon $T$, trade-off parameter $\eta$, and matrices $\kappa_j, \nu_j, \Sigma_j$
\STATE
Initialize $w_1 = (\frac{1}{M}, \dots, \frac{1}{M}) \in \mathbb{R}^{M}$ 
\FOR {$t = 1$ {\bfseries{to}} $T$ }
    \STATE  $\alpha_t =\mathrm{BEST}(w_t)$
    \FOR {$j = 1$ {\bfseries{to}} $M$ }
        \STATE Set $w_{t+1, j} \propto  w_{t, j} \exp\left \{\eta ( \kappa_j - 2 \nu_j^\top \alpha_{t} + \alpha^{\top}_{t} \Sigma_j \alpha_{t}) \right \}$
    \ENDFOR
\ENDFOR
\STATE {\bfseries{Output: }} $\frac{1}{T} \sum_{i = 1}^T w_t$
\end{algorithmic}
\end{algorithm}

We apply the regret guarantee of Multiplicative Weights (see e.g.\ \cite{cesabianchi1997expert}) to obtain the following. 
\begin{lemma}[Adversary's Regret]\label{lem:mw-linear}
There exists a constant $C \in \mathbb R$ such that 
running Multiplicative Weights for $w_t$, with $\eta = \sqrt{\frac{\log M}{T}}$, yields the following regret:
\[
    \max_{w \in \Delta(M)}  \sum_{t = 1}^T \mathcal L(\alpha_t, w) -
    \sum_{t = 1}^T \mathcal L(\alpha_t, w_t) \leq 
    C\sqrt{T \log M }.
\]
\end{lemma}

The proof of Theorem \ref{thm:approx-sol-linear} follows from simply applying the guarantees of Lemma \ref{lem:best-response-linear} and Lemma \ref{lem:mw-linear} into Theorem \ref{thm:approx-equilibrium}.

\subsection{Reproducing Kernel Hilbert Spaces}\label{app:rkhs-spaces}
Let $\cH$ be a Reproducing Kernel Hilbert space (RKHS) of bounded norm with kernel $K_\cH : \cX \times \cX \to \mathbb R$ and whose norm is bounded above by some $A \in \mathbb R$. Suppose that $\cF = \cH$. Since $\cH$ is a closed linear space, the classes $\cH$ and $\cH - \cH$ refer to the same RKHS and hence, $\cF \supseteq \shull(\cH - \cH)$. For a hypothesis $h \in \cH$ and index $j \in [M]$, define the moment violation
\[ 
L_{j,n}(h, f):=  \bE_{j, n} [2 (Y - h(X))f(X) - f(X)^2 ] .
\]

To solve the empirical (regularized) min-max adversarial moment problem, we first rewrite the objective. 
\begin{lemma}[Adversarial Moment Problem, Restated for RKHS]\label{lem:rkhs-restated-appendix}
Let $X_{ij}$  and $Y_{ij}$ denote the $i$-th sample from the $j$-th dataset and let $y_j = (Y_{1j}, Y_{2j}, \dots, Y_{nj})$. 
Let $\cK_j$ denote the empirical kernel matrix for the $j$-th dataset with $\cK_{j, ii'} = K(X_{ij}, X_{i'j})$.
Furthermore, let $\cM_{jj'}$ denote the $n \times n$ matrix, with entries $\cM_{jj',ii'}= K(X_{ij}, X_{i'j'})$. Let $\cM_j = [\cM_{j1}, \ldots, \cM_{jM}]$ be the $n\times (n\cdot M)$ concatenation of these matrices along the column axis, and $\cM=[\cM_1; \ldots; \cM_M]$ be the $(n\cdot M)\times (n\cdot M)$ concatenation of $\cM_j$ along the row axis. Then
\begin{align*}
&\min_{h \in \cH} \max_{j \in [M]} \max_{f \in \cF} L_{j,n}(h,f) - \lambda \|f \|^2 + \nu \|h \|^2 \\
=~ & \min_{\substack{\alpha \in \mathbb R^{Mn} \\ \|\alpha\| \leq A}} \max_{w \in \Delta(M)} \frac{1}{n} \sum w_j \left ( \kappa_j - 2 \nu_j^\top \alpha + \alpha^\top \Sigma_j \alpha \right )
\end{align*}
where
\begin{align*}
\kappa_j :=~& y_j^\top Q_j y_j &
\nu_j :=~&  \cM_j^\top Q_j^\top  y_j\\
\Sigma_j :=~& \cM_j^\top Q_j\cM_j + \mu\,n\,\cM &
Q_j :=~& \cK_j (\cK_j + \lambda n I)^{-1}
\end{align*}
\end{lemma}

\begin{proof}[Proof of Lemma \ref{lem:rkhs-restated-appendix}]
Fix a hypothesis $h \in \cH$ and $j \in [M]$. By the generalized representer theorem for an RKHS, due to \citet{scholkopf2001generalized}), the optimal solution to the moment violation
\[
\max_{f \in \cF} L_{j, n}(h, f) = \max_{f \in \cF} \bE_{j,n}[ 2(Y - h(X)) f(X) - f(X)^2] - \lambda \|f\|^2 + \mu \|h\|^2
\]
will be attained by a function $f_j$ of the form 
$
f_j(\cdot) = \sum_{i = 1}^{n} \beta_{ij}K(X_{ij}, \cdot),
$
where $\beta_{ij} \in \mathbb R$ for every $i \in [n]$.
Thus, we can rewrite the moment violation as follows:
\[
\max_{f \in \cF} L_{j, n}(h, f) = \max_{\beta_j} \frac{1}{n} (2 \rho_j(h)^\top \cK_j \beta_j - \beta_j^\top \cK_j \cK_j \beta_j) - \lambda \beta_j^\top \cK_j \beta_j
+ \mu \|h\|^2,
\]
where $\rho_j(h) = (Y_{1j} - h(X_{1j}), \dots, Y_{nj} - h(X_{nj}))$ and $\beta_j = (\beta_{1j}, \cdots \beta_{nj})$. 

We make use of the following lemma to further restate the moment violation.
\begin{lemma}\label{lem:opt-violation-rkhs-appendix}
Fix a hypothesis $h \in \cH$ and dataset $j \in [M]$. The optimal moment violation $\max_{f \in \cF} L_{j, n}(f)$ is 
\begin{align*}
    &\frac{1}{n} \rho_j(h)^\top Q_j \rho_j(h) + \mu \|h \|^2,
    & Q_j := & ~ \cK_j (\cK_j + n \lambda I)^{-1}
\end{align*}
and is attained when 
\[
\beta_j = (\cK_j + n\lambda I)^{-1} \rho_j(h).
\]
\end{lemma}
\begin{proof}
The first order condition tells us that
\begin{align*}
0 &= \frac{2}{n} \cK_j \rho_j(h)  - \frac{2}{n} \cK_j\cK_j \beta_j - 2 \lambda \cK_j \beta_j.
\end{align*}
Assuming that $\cK_j$ and $\cK_j + n\lambda I$ are invertible, we find that
\begin{align*}
\beta_j 
&= (\cK_j^2 + n\lambda \cK_j)^{-1} \cK_j \rho_j(h) \\
&= (\cK_j (\cK_j + n\lambda I))^{-1} \cK_j \rho_j(h) \\
&= (\cK_j + n\lambda I)^{-1} \rho_j(h)
\end{align*}
Substituting this value $\beta_j$ into our objective, and setting $Q_j = \cK_j(\cK_j + n \lambda I)^{-1}$ we derive
\begin{align*}
\frac{2}{n} \rho_j(h)^\top \cK_j \beta_j 
&= \frac{2}{n} \rho_j(h)^\top \cK_j  (\cK_j + n\lambda I)^{-1} \rho_j(h) \\
&= \frac{2}{n} \rho_j(h)^\top Q_j \rho_j(h) \\
- \left (\frac{1}{n} \beta_j^\top \cK_j \cK_j \beta_h + \lambda \beta_j^\top \cK_j \beta_j  \right ) 
&= - \frac{1}{n} \beta_j^\top (  \cK_j + n\lambda I )\cK_j \beta_j \\
&= - \frac{1}{n} \rho_j(h)^\top (\cK_j + n \lambda I )^{-1} (  \cK_j + n\lambda I )\cK_j  (\cK_j + n\lambda I)^{-1} \rho_j(h) \\
&= - \frac{1}{n} \rho_j(h)^\top \cK_j  (\cK_j + n\lambda I)^{-1} \rho_j(h) \\
&= - \frac{1}{n} \rho_j(h)^\top Q_j \rho_j(h)
\end{align*}
Hence, it follows that 
\begin{align*}
    \max_{f \in \cF} L_{j, n}(f) = \frac{1}{n} \rho_j(h)^\top Q _j\rho_j(h) + \mu \|h \|^2 
\end{align*}
\end{proof}
\noindent With the claim above, we have in fact  shown that 
\begin{align*}
\min_{h \in \cH} \max_{j \in [M]}\max_{f \in \cF} L_{j, n}(f)
&= \min_{h \in \cH} \max_{j \in [M]}\frac{1}{n} \rho_j(h)^\top Q_j \rho_j(h) + \mu \|h \|^2.
\end{align*}
By the generalized representer theorem, we also know that the optimal hypothesis $h \in \cH$ will be of the form
\begin{align*}
    h(\cdot) = \sum_{j = 1}^{M} \sum_{i = 1}^{n} \alpha_{ij}K(X_{ij}, \cdot) 
\end{align*}
where $\alpha_{ij} \in \mathbb R$ for every $i \in [n]$ and every $j \in [M]$.
We write
$\alpha_j = (\alpha_{1j}, \dots, \alpha_{nj})$ and $\alpha = (\alpha_1, \dots, \alpha_M) \in \mathbb R^{nM}$. 
Recalling the definition of the matrix $\cM_j$, we find that
$
(h(X_{1j}), \dots, h (X_{nj}) ) = \cM_j \alpha 
$ and hence, $\rho_j(h) = y_j - \cM_j \alpha$. Recalling the definition of $\cM$, we also have $\|h\|^2 = \alpha^\top \cM \alpha$.  Thus, we derive that
\begin{align*}
&\min_{h \in \cH} \max_{j \in [M]} \max_{f \in \cF} L_{j, n}(f)\\
= ~ & \min_{\substack{\alpha \in \mathbb R^{Mn} \\ \|\alpha\| \leq A}} \max_{j \in [M]} 
\frac{1}{n}
\left (
    y_j^\top Q_j y_j 
    - 2 y_j^\top Q_j \cM_j \alpha  
    + \alpha^\top \cM_j^\top Q_j \cM_j \alpha
\right ) 
+ \mu \alpha^\top \cM \alpha  \\
= ~ & \min_{\substack{\alpha \in \mathbb R^{Mn} \\ \|\alpha\| \leq A}} \max_{j \in [M]} \frac{1}{n}
\left (
    y_j^\top Q_j y_j 
    - 2 y_j^\top Q_j \cM_j \alpha  
    + \alpha^\top \cM_j^\top Q_j \cM_j \alpha 
    + n\mu \alpha^\top \cM \alpha 
\right ) \\
= ~ & \min_{\substack{\alpha \in \mathbb R^{Mn} \\ \|\alpha\| \leq A}} \max_{j \in [M]} \frac{1}{n}\left (\kappa_j - 2 \nu_j^\top \alpha  + \alpha^\top \Sigma_j \alpha  \right ) 
\end{align*}
where 
\begin{align*}
\kappa_j :=~& y_j^\top Q_j y_j &
\nu_j :=~&  \cM_j^\top Q_j^\top  y_j\\
\Sigma_j :=~& \cM_j^\top Q_j\cM_j + \mu\,n\,\cM.
\end{align*}

Our desired result follows from the definition of $\cH$ and the convexity of $\Delta(M)$
\end{proof}

We therefore end up with an optimization problem that resembles the one in the linear setting. The proof of \ref{thm:approx-sol-rkhs} beyond this point is identical to the linear case. See Section \ref{app:linear-spaces} for details on the no-regret dynamics.

\subsection{Convex Hypothesis Spaces}\label{app:convex-spaces}
Consider a convex hypothesis class $\cH$ and assume that $\cF = \shull(\cH - \cH)$. Note that $\cF = \shull(\cH - \cH)$ is also a convex hypothesis space. With our assumptions on $\cH$ and $\cF$, we can rewrite the min-max adversarial moment problem.
For a hypothesis $h \in \cH$ and index $j \in [M]$, define the moment violation
\[ 
V_j(h):= \max_{f \in \cF} \bE_{j, n} [2 (Y - h(X))f(X) - f(X)^2 ].
\]

Let $X_{ij}, Y_{ij}$ denote the $i$-th sample from the $j$-th dataset. 
For each hypothesis $h \in \cH$ and test function $f \in \cF$ and dataset $j$, 
we consider vectors $\alpha_j = (\alpha_{1j}, \dots, \alpha_{nj}), \beta_j = (\beta_{1j}, \dots, \beta_{nj})$, where $\alpha_{ij} = h(X_{ij}), \beta_{ij} = f(X_{ij})$. If we let $A, B$ denote the spaces containing the vectors $\alpha, \beta$, then we reduce the min-max adversarial moment problem to an $n$-dimensional convex-concave optimization problem:

\[
\min_{h \in \cH} \max_{j \in [M]} V_j(h) = \min_{\alpha \in A} \max_{j \in [M]} \max_{\beta \in B} 2(y_j - \alpha_j)^\top \beta_j - \|\beta_j\|_2^2
\]

We find an approximate solution to our desired problem using no-regret dynamics, where the Learner and Adversary are respectively trying to minimize and maximize the objective
\[
\mathcal L(\alpha, w, \beta) := \sum_{j = 1}^M w_j ( 2(y_j - \alpha_j)^\top \beta_j - \|\beta_j\|_2^2).
\]
The Learner will play a vector $\alpha \in A$ and the Adversary will play a tuple $(w, \beta) \in \Delta(M) \times B$.

\subsubsection*{The Learner's Best Response}
In each round $t$, the Learner needs to best respond to the Adversary's action $(w_t, \beta_t)$ by choosing an action $\alpha_t \in \argmin_{\alpha \in A} \mathcal L(\alpha, w_t, \beta_t)$. Note that 
\begin{align*}
\argmin_{\alpha \in A} \mathcal L(\alpha, w_t, \beta_t) 
&= \argmin_{\alpha \in A}  - \sum_{j = 1}^{M} w_{t,j} \alpha_j^\top \beta_{t, j} 
\end{align*}

The Learner solves the problem above by making use of a linear optimization oracle, $\mathrm{ORACLE}_\cH$, which computes
\begin{align*}
  \mathrm{ORACLE}_\cH(X, \beta_t, w_t) \in
  \argmin_{h \in \cH} - \sum_{j = 1}^M w_{t, j} \bE_{j,n}[h(X_{ij}) \beta_{t, ij})]
\end{align*}

\begin{algorithm}[ht!] 
\caption{Best Response, $\mathrm{BEST}(w, \beta)$, for Learner}
\label{alg:best-response-convex}
\begin{algorithmic}
\STATE {\bfseries{Input: }} vectors $w, \beta$ played by Adversary

\STATE Let $h = \mathrm{ORACLE}_\cH(X, \beta_t, w_t)$ \\
\FOR {$i = 1$ {\bfseries{to}} $n$} 
\FOR {$j = 1$ {\bfseries{to}} $M$} 
\STATE Set $\alpha_{ij} = h(X_{ij})$
\ENDFOR
\ENDFOR

\STATE {\bfseries{Output: }} $\alpha$
\end{algorithmic}
\end{algorithm}

\noindent Based on the discussion above, the following lemma is immediate.

\begin{lemma}\label{lem:best-response-convex}
Let $(w_1, \beta_1), \dots, (w_T, \beta_T) \in \Delta(M) \times B$ be some sequence, and let $\alpha_1, \dots, \alpha_T \in A$ be a sequence such that 
$\alpha_t = \mathrm{BEST}(w_t, \beta_t)$ for each $t \in [T]$.
Then 
\[
\sum_{t = 1}^T \mathcal L(\alpha_t, w_t, \beta_t)
- \min_{\alpha \in A }\sum_{t = 1}^T \mathcal L (\alpha, w_t, \beta_t) \leq 0.
\] 
\end{lemma}

\subsubsection*{The Adversary's No-regret Updates}
In each round $t$, the Adversary's play consists of a vector $\beta_t \in B$ and a weight $w_t \in \Delta(M)$. The weight $w_t$ is updated according to the Multiplicative Weights algorithm. The vector $\beta_t$ is updated according to the Follow the Leader algorithm. The Adversary maintains a vector $\beta_{t,j}$ for each dataset $j \in [M]$ which is updated as follows: 
\begin{align*}
\beta_{t+1, j} 
&= \arg \min_{\beta_j \in B_j} \frac{1}{t} \sum_{\tau = 1}^t (2(y_j - \alpha_{t,j})^\top \beta_j - \| \beta_j \|_2^2) \\
&=  \arg \min_{\beta_j \in B_j} \| y_j - \overline{\alpha}_{t, j} - \beta_j \|_2^2
\end{align*}

The Adversary solves the optimization problem above by making use of a regression oracle, $\mathrm{ORACLE}_\cF$, which computes 
\[
\mathrm{ORACLE}_\cF(X, Y, \overline{\alpha}_{t, j}) \in 
\argmin_{f \in \cF} \bE_{j,n}\left [ ( Y - \overline{\alpha}_{t, j} - f(X_{ij}))^2\right]
\]
for more details on the no-regret dynamics, refer to Algorithm \ref{alg:best-response-convex} .

\begin{algorithm}[ht]
\caption{No-Regret Dynamics}
\label{alg:-no-regret-oracle}
\begin{algorithmic}
\STATE {\bfseries{Input:}} time horizon $T$, trade-off parameter $\eta$,
\STATE Initialize $\beta_1$ to some arbitrary vector in $B$ 
\STATE Initialize $w_1 \gets (\frac{1}{M}, \dots, \frac{1}{M}) \in \mathbb{R}^{M}$ 
\FOR {$t= 1$ {\bfseries{to}} $T$}
    \STATE $\alpha_t \gets \mathrm{BEST}(w_t, \beta_t)$ 
    \STATE $\overline{\alpha_t} \gets \frac{1}{\tau} \sum_{\tau = 1}^{t} \alpha_\tau$ 

    \FOR{$j = 1$ {\bfseries{to}} $M$}
        \STATE Let $f_j = \mathrm{ORACLE}_\cF(X, Y, \overline{\alpha_t})$\\
        \STATE Set $\beta_{t+1, j} = (f_j(X_{1j}), \dots, f_j(X_{nj}))$
    \ENDFOR
    
   \FOR{$j = 1$ {\bfseries{to}} $M$}
        \STATE Set $w_{t+1, j} \propto  w_{t, j} \exp\left \{ \eta \left( 2(y_j - \alpha_{t,j})^\top \beta_{t+1,j} - \| \beta_{t+1, j} \|_2^2 \right )  \right \}$
    \ENDFOR
    
\ENDFOR

\algorithmicensure{$\frac{1}{T} \sum_{i = 1}^T w_t$, }
\end{algorithmic}
\end{algorithm}

\begin{lemma}[Adversary's Regret]\label{lem:convex-regret-appendix}
Let $\alpha_1, \dots, \alpha_T \in A$ be some sequence.
There exists a constant $C$ such that running Multiplicative Weights for $w_t$ with $\eta = \sqrt{\frac{\log M}{T}}$ and running follow the leader for $\beta_t$ yields the following regret
\[
\max_{w} \max_{\beta} \sum_{t = 1}^T \mathcal L (\alpha_t, w, \beta)
- \sum_{t = 1}^T \mathcal L (\alpha_t, w_t, \beta_t) 
\leq C \left ( \log T + \sqrt{T \log M} \right )
\]
\end{lemma}

\begin{proof}[Proof of Lemma \ref{lem:convex-regret-appendix}]
\begin{align*}
\sum_{t = 1}^{T} \mathcal L (\alpha_t, w_t, \beta_t)
&=\sum_{t=1}^T \sum_{j=1}^M w_{t,j} \left(2 (y_j - \alpha_{t,j})^\top\beta_{t,j} - \|\beta_{t,j}\|_2^2\right) \\
&\geq \max_{w \in \Delta(M)} \sum_{t=1}^T w_j \left(2 (y_j - \alpha_{t,j})^\top\beta_{t,j} - \|\beta_{t,j}\|_2^2\right) - C_1 \sqrt{ T \log M}  \\ 
&\geq \max_{w \in \Delta(M)} \max_{\beta_j\in B_j} \sum_{t=1}^T \left(2 (y_j - \alpha_{t,j})^\top\beta_j - \|\beta_j\|_2^2\right) - C_1 \sqrt{ T \log M} - C_2 \log T 
\end{align*}

The first inequality follows from the regret guarantee of Multiplicative Weights (see e.g.\ \cite{cesabianchi1997expert}). The second follows from the regret guarantee of follow the leader (cf. \cite{hazan2007logarithmic}; \cite{shalev2012online} for more details ).
\end{proof}

\section{Missing Details from Section~\ref{sec:exp}}
\subsection{Missing Details from Section~\ref{subsec:exp-rkhs}}\label{app:synthetic-details}

We include a python notebook containing the code used to generate the synthetic data experiments in our supplementary code (\texttt{RKHS\_example\_final.ipynb}). 

Figure~\ref{fig:data-dist} illustrates the details of our data generation process where groups follow one of two parabolic curves with either high or low additive noise. 

\begin{figure}[hb]
    \centering
\includegraphics[scale=0.7]
{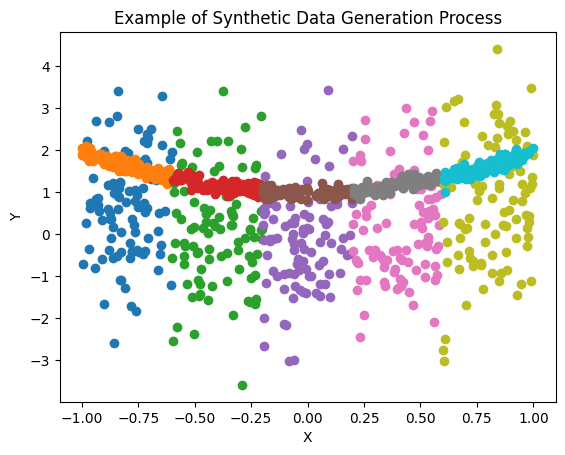}
    \caption{Example of synthetic data distribution for 10 groups, each of size 100. Each group is color-coded for easy differentiation. Every group follows a parabolic function with additive gaussian noise.}
    \label{fig:data-dist}
\end{figure}

In our experiments, we observe a tradeoff between group size and number of groups in terms of the runtime of our approach vs the baselines. We find that in regimes with small numbers of groups, each very large in size, MRO may out-perform our approach in terms of runtime (see Figure~\ref{fig:bad-regime}). It's possible that this gap could be improved through further optimization in future work. 

\begin{figure}
    \centering
\includegraphics[scale=0.7]
{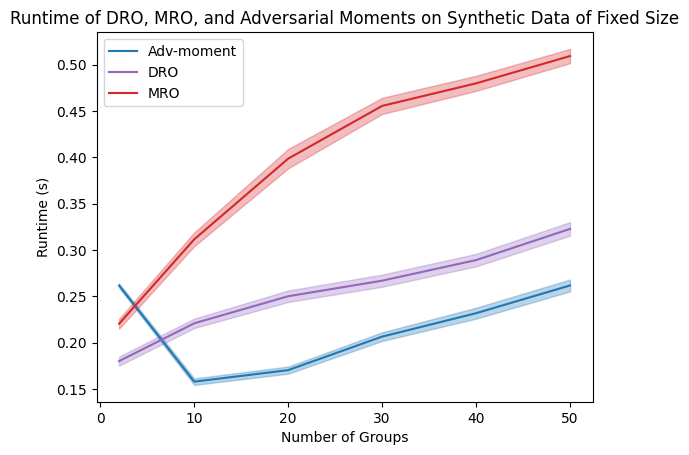}
    \caption{Runtime of our method (Adv-moment) and baselines (groupDRO (DRO), MRO) on a synthetic dataset of size 2000 as the number of groups increase. We see that when there are only two groups each of size 1000, MRO outperforms the runtime of the adversarial moment method, but this reduces as the number of groups is increased.}
    \label{fig:bad-regime}
\end{figure}

\subsection{Missing Details from Section~\ref{subsec:exp-neural}}
\label{app:exp-neural}
\textbf{Loss Function} In the experiments, we used cross entropy loss for both MRO and groupDRO. It is quite common to use cross entropy loss for a deep learning task like CelebA. However, we note that our results are pretty comparable to MRO and groupMRO despite us not having optimized the hyper-parameters for square loss and borrowing the hyper-parameters used for groupDRO and MRO.

\textbf{Hyper-parameters}
For the stochastic gradient descent update, we used a bath size of 32 while using the best hyperpameters used for groupDRO (best hyperparamets used for strong $l_2$ regularization\footnote{learning rate=1e-5, $l_2$ penalty=0.1, weight decay=0.1}); we used the same hyperparameters for updating the adversary and set $a_n=1$. Although we used a different batch size and trained the network for a smaller number of epochs (\citet{sagawa2019distributionally} used a batch size of 128 and trained for 50 epochs), we note that the results we get for groupDRO are comparable to, if not better, the results reported in \citet{sagawa2019distributionally}.

\paragraph{Adversary Network Architecture}
For the adversary network, to output 10 numbers as opposed to 1. Then, we concatenate $j \in [M]$ to these 10 values and go through a 100 hidden node layer with a ReLU activation to get the final output value. 

\paragraph{Stochastic Gradient Optimization}
At each period $t$, we draw $B$ samples from each subpopulation and let $\cB_{t,j}$ denote the set of samples drawn from population $j$ in period $t$. Then we can construct the online stochastic gradient descent-ascent analogue of the above no-regret dynamics as:
\begin{align*}
    \beta_{t+1} =~& \beta_{t} + \eta_j \, \frac{1}{M B} \sum_{j \in [M]}\sum_{i\in \cB_{t,j}} U_{t,j,i}\\
    w_{t+1,j} \propto~& w_{t,j}\, \exp\left\{\eta_w\, \frac{1}{B} \sum_{i\in \cB_{t,j}} U_{t,j,i}\right\}\\
    \alpha_{t+1} =~& \alpha_t - \eta_\alpha\, \sum_{j=1}^M w_{t,j} \frac{1}{B} \sum_{i\in \cB_{t,j}} V_{t,j,i}.
\end{align*}
We further center the adversary around the model as discussed before, but this centering is handled  as PyTorch handles the gradient computation automatically.

\end{document}